\documentclass[letterpaper]{article} 

\newif\ifshowappendix
\showappendixtrue

\usepackage{aaai25}  
\usepackage{times}  
\usepackage{helvet}  
\usepackage{courier}  
\usepackage[hyphens]{url}  
\usepackage{graphicx} 
\usepackage{natbib}  
\usepackage{caption} 
\usepackage{complexity}
\usepackage{bm}         
\usepackage{amssymb}
\usepackage{amsthm}
\usepackage{amsmath}
\usepackage{enumitem}
\usepackage{caption}    
\usepackage{subcaption} 
\usepackage{centernot}
\usepackage{complexity}
\usepackage{mathrsfs}   
\usepackage{tabularx}   
\usepackage{multirow}   
\usepackage{array}      
\usepackage{booktabs}   
\usepackage{colortbl}   
\usepackage{xargs}      
\usepackage{tikz}
\usepackage{tikz-cd}
\usepackage{url}
\usepackage{makecell}   
\usepackage{float}
\usepackage{etoc}
\usepackage[ruled,noline,linesnumbered,noend]{algorithm2e} 
\usepackage{lipsum}
\usepackage{float}
\usepackage{adjustbox}
\usepackage{placeins}
\usepackage{listings}
\usepackage{cleveref}
\usepackage{threeparttable}

\usetikzlibrary{calc,positioning}

\newtheorem{theorem}{Theorem}[section]

\newtheorem{lemma}[theorem]{Lemma}
\newtheorem{proposition}[theorem]{Proposition}
\theoremstyle{definition}
\newtheorem{example}[theorem]{Example}
\theoremstyle{definition}
\newtheorem{definition}[theorem]{Definition}

\SetInd{0.5em}{0em}

\newcolumntype{Y}{>{\raggedleft\arraybackslash}X}

\newcommand{\zerocell}[1]{-}

\newcommand{\bestbest}[1]{\cellcolor{blue!30}{#1}}
\newcommand{\poss}[1]{\cellcolor{green!30}{#1}}
\newcommand{\negg}[1]{\cellcolor{red!30}{#1}}
\newcommand{\possA}[1]{\poss{#1}}

\definecolor{caribbeangreen}{rgb}{0.0, 0.8, 0.6}
\definecolor{brilliantlavender}{rgb}{0.96, 0.73, 1.0}
\definecolor{amethyst}{rgb}{0.6, 0.4, 0.8}
\definecolor{ao(english)}{rgb}{0.0, 0.5, 0.0}
\definecolor{arylideyellow}{rgb}{0.91, 0.84, 0.42}
\definecolor{asparagus}{rgb}{0.53, 0.66, 0.42}
\definecolor{aquamarine}{rgb}{0.5, 1.0, 0.83}
\definecolor{babyblue}{rgb}{0.54, 0.81, 0.94}
\definecolor{fwtchanged}{rgb}{0.3, 0.3, 0.7}
\definecolor{rosewood}{rgb}{0.4, 0.0, 0.04}
\definecolor{oldmauve}{rgb}{0.4, 0.19, 0.28}
\definecolor{myrtle}{rgb}{0.13, 0.26, 0.12}
\definecolor{magenta(dye)}{rgb}{0.79, 0.08, 0.48}

\definecolor{plta}{rgb}{0.12, 0.47, 0.71}
\definecolor{pltb}{rgb}{   1, 0.5, 0.05}
\definecolor{pltc}{rgb}{0.17, 0.63, 0.17}
\definecolor{pltd}{rgb}{0.84, 0.15, 0.16}

\newcommand{\colourdc}{blue}
  
\newcommand{\todocustom}[3]{\todo[linecolor=#2,backgroundcolor=#2!25,bordercolor=#2,size=\tiny,#3]{#1}}  
\newcommandx{\nbdc}[2][1=]{\todocustom{#2}{\colourdc}{#1}}

\def\N{\mathbb{N}}
\def\R{\mathbb{R}}

\renewcommand{\phi}{\varphi}

\def\la{\leftarrow}

\newcommand{\abs}[1]{\left| #1 \right|}

\newcommand{\gen}[1]{\left< #1 \right>}

\newcommand{\set}[1]{\left\{ #1 \right\}}
\newcommand{\seta}[1]{\{ #1 \}}

\newcommand{\mseta}[1]{\{ \!\! \{ #1 \} \!\! \}}

\newcommand{\lr}[1]{\left( #1 \right)}

\newcommand{\range}[1]{\left[\!\left[ #1 \right]\!\right]}

\newcommand{\problem}{\mathbf{P}}
\newcommand{\objects}{\mathcal{O}}
\newcommand{\predicates}{\mathcal{P}}

\newcommand{\schemata}{\mathcal{A}}
\newcommand{\goal}{\mathcal{G}}
\newcommand{\objs}{\mathbf{o}}

\newcommand{\arity}{\mathrm{ar}}

\newcommand{\hash}{\textsc{hash}}
\newcommand{\graphFont}[1]{\mathbf{#1}}
\newcommand{\neighbour}{\graphFont{N}}
\newcommand{\featCat}{\graphFont{F}}

\newcommand{\featEdge}{\graphFont{L}}
\newcommand{\graph}{\graphFont{G}}
\newcommand{\nodes}{\graphFont{V}}
\newcommand{\edges}{\graphFont{E}}

\newcommand{\nodeCat}{\Sigma_{\text{V}}}
\newcommand{\edgeCat}{\Sigma_{\text{E}}}

\newcommand{\multiset}{\textbf{M}}

\newcommand{\qb}{\mathrm{QB}}
\def\neww{\text{new}}
\def\oldd{\text{old}}
\def\new{F_{\neww}}
\def\old{F_{\oldd}}
\newcommand{\WL}{\mathrm{WL}}
\newcommand{\AT}{\mathrm{AT}}
\newcommand{\plannerFont}[1]{$\texttt{#1}$}
\newcommand{\atoms}{\plannerFont{at}}
\newcommand{\wl}{\plannerFont{wl}}
\newcommand{\atomswl}{\plannerFont{at;wl}}
\newcommand{\features}{\mathcal{F}}
\newcommand{\graphRep}{\mathcal{G}}
\newcommand{\symm}{\sim_U}
\newcommand{\hff}{$h^{\mathrm{ff}}$}
\newcommand{\hadd}{$h^{\mathrm{add}}$}
\newcommand{\hgc}{$h^{\mathrm{gc}}$}

\nocopyright

\title{Symmetry-Invariant Novelty Heuristics\\via Unsupervised Weisfeiler-Leman Features}
\author{
  Dillon Z. Chen
}
\affiliations{
  LAAS-CNRS, University of Toulouse
}

\begin{document}

\maketitle

\begin{abstract}
  Novelty heuristics aid heuristic search by exploring states that exhibit novel atoms. However, novelty heuristics are not symmetry invariant and hence may sometimes lead to redundant exploration. In this preliminary report, we propose to use Weisfeiler-Leman Features for planning (WLFs) in place of atoms for detecting novelty. WLFs are recently introduced features for learning domain-dependent heuristics for generalised planning problems. We explore an unsupervised usage of WLFs for synthesising lifted, domain-independent novelty heuristics that are invariant to symmetric states. Experiments on the classical International Planning Competition and Hard To Ground benchmark suites yield promising results for novelty heuristics synthesised from WLFs.
\end{abstract}

%

\section{Introduction}
Novelty heuristics~\cite{lipovetzky.geffner.aaai2017,katz.etal.icaps2017}, which originated from the notion of width for propositional planning~\cite{lipovetzky.geffner.ecai2012}, are a simple yet powerful domain-independent technique for enhancing any given heuristic by helping guide exploration in search.
Novelty heuristics have later been extended to handle numeric planning~\cite{ramirez.etal.aamas2018,teichteilkoenigsbuch.etal.ijcai2020,chen.thiebaux.socs2024}, lifted planning~\cite{correa.seipp.icaps2022}, to guide policy search in generalised planning~\cite{lei.etal.socs2023}, and enhancing the LAMA planner~\cite{correa.seipp.c2024}.

Weisfeiler-Leman Features for planning (WLFs) are a recently introduced, efficient-to-compute feature generator for planning problems~\cite{chen.etal.icaps2024}.
We introduce the usage of WLFs for generating domain-independent novelty heuristics, which we name \emph{Weisfeiler-Leman (WL) novelty heuristics}.
We do so by generalising the quantified-both novelty heuristic framework introduced by \citet{katz.etal.icaps2017} to support arbitrary planning feature generators.
We show that WL novelty heuristics are symmetry-invariant, a first of its kind for novelty heuristics.
Preliminary experimental evaluation on the International Planning Competition and Hard To Ground benchmark suites show promising results for WL novelty heuristics.

\section{Background}
\label{sec:background}
This section formalises lifted STRIPS problems, the Weisfeiler-Leman Features for planning, and Quantified Both novelty heuristics.
Let $\range{n}$ denote the set of integers $\set{1,\ldots,n}$ and $\abs{S}$ the size of a set $S$.
The minimum of an empty set is $\infty$.

\paragraph{Planning Problem and Lifted Representation}
A classical planning problem is a deterministic state transition model~\cite{geffner.bonet.2013} given by a tuple $\problem = \gen{S, A, s_0, G}$ where $S$ is a set of states, $A$ a set of actions, $s_0 \in S$ an initial state, and $G \subseteq S$ a set of goal states.
Each action $a \in A$ is a function $a: S \rightarrow S \cup \set{\bot}$ where $a(s) = \bot$ if $a$ is not applicable in $s$, and $a(s) \in S$ is the successor state when $a$ is applied to $s$.
A solution for a planning problem is a plan: a sequence of actions $\pi = a_1, \ldots, a_n$ where $s_i = a_i(s_{i-1}) \not= \bot$ for $i \in \range{n}$ and $s_n \in G$.
A state $s$ in a planning problem $\problem$ induces a new planning problem $\problem' = \gen{S, A, s, G}$.
A planning problem is solvable if there exists at least one plan.
Satisficing planning refers to the task of finding any plan for a planning problem if it exists.

Planning problems are often compactly formalised in a lifted representation using predicate logic, such as via PDDL~\cite{ghallab.etal.1998,haslum.etal.2019}.
More specifically, a \emph{lifted (STRIPS) planning problem} is a tuple $\problem = \gen{\objects, \predicates, \schemata, s_0, \goal}$, where $\objects$ denotes a set of objects, $\predicates$ a set of predicate symbols, $\schemata$ a set of action schemata, $s_0$ the initial state, and $\goal$ now the goal condition.
Understanding of the transition system induced by $\schemata$ is not necessary for this paper, as the proposed novelty heuristics are applicable when the action model is not known~\cite{frances.etal.ijcai2017}.
We next focus on state and goal condition representations.

Each symbol $P \in \predicates$ is associated with an arity $\arity(P)\in\N \cup \set{0}$.
Predicates take the form $P(x_1,\ldots,x_{\arity(P)})$, where the $x_i$s denote their arguments.
Atoms are defined by substituting objects into predicate arguments, e.g. $p = P(o_1, \ldots, o_{\arity(P)})$.
More specifically, given $P \in \predicates$, and a tuple of objects $\objs = \gen{\objs_1,\ldots,\objs_{\arity(P)}}$, we denote $P(\objs)$ as the atom defined by substituting $\objs$ into arguments of $P$.
A state $s$ is a set of atoms.
The goal condition $\goal$ also consists of a set of atoms, and a state $s$ is a goal state if $s \supseteq \goal$.

We introduce the notational shorthand $\problem[\omega]$ as the $\omega$ component of a problem $\problem$; e.g. $\problem[s_0]$ is the initial state of $\problem$; and given a state $s$, let $\problem_s = \gen{\problem[\objects], \problem[\predicates], \problem[\schemata], s, \problem[\goal]}$.

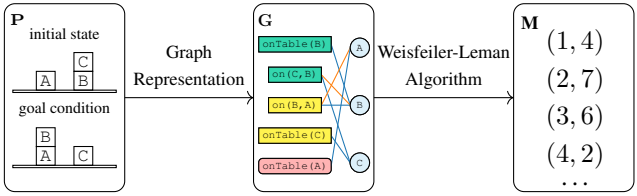
\begin{figure}[t]
  \centering
  \newcommand{\segmentWidth}{1.6cm}
  \newcommand{\segmentHeight}{2.5cm}
  \newcommand{\segmentShift}{3.5125cm}
  \newcommand{\trainColour}{cyan}
\newcommand{\objdiff}{0.6}
\newcommand{\factdiff}{0.6}
\newcommand{\xsc}{1.2}
\newcommand{\xscb}{1.7}
\newcommand{\ysc}{2.3}
\newcommand{\objcolour}{babyblue!30}
\newcommand{\fmt}[1]{$\texttt{#1}$}
\newcommand{\nilgTikz}{
    \large
    \node[draw, circle, fill=\objcolour] (A) at
    (\ysc,+2*\factdiff*\xscb) {\fmt{A}};
    \node[draw, circle, fill=\objcolour] (B) at
    (\ysc,+0*\factdiff*\xscb) {\fmt{B}};
    \node[draw, circle, fill=\objcolour] (C) at
    (\ysc,-2*\factdiff*\xscb) {\fmt{C}};
    \node[draw, rectangle, fill=caribbeangreen!80] (a) at
    (0,3*\factdiff*\xsc) {\fmt{onTable(B)}};
    \node[draw, rectangle, fill=caribbeangreen!80] (b) at
    (0,1.5*\factdiff*\xsc) {\fmt{on(C,B)}};
    \node[draw, rectangle, fill=yellow!80] (c) at
    (0,0*\factdiff*\xsc) {\fmt{on(B,A)}};
    \node[draw, rectangle, fill=yellow!80] (d) at
    (0,-1.5*\factdiff*\xsc) {\fmt{onTable(C)}};
    \node[draw, rounded corners=1.5, fill=red!30] (e) at
    (0,-3*\factdiff*\xsc) {\fmt{onTable(A)}};

    \path [-,draw=plta] (B.west) edge (a.east);

    \path [-,draw=plta] (C.west) edge (b.east);
    \path [-,draw=pltb] (B.west) edge (b.east);

    \path [-,draw=plta] (B.west) edge (c.east);
    \path [-,draw=pltb] (A.west) edge (c.east);

    \path [-,draw=plta] (C.west) edge (d.east);

    \path [-,draw=plta] (A.west) edge (e.east);

    \node[draw, circle, fill=\objcolour] (A) at
    (\ysc,+2*\factdiff*\xscb) {\fmt{A}};
    \node[draw, circle, fill=\objcolour] (B) at
    (\ysc,+0*\factdiff*\xscb) {\fmt{B}};
    \node[draw, circle, fill=\objcolour] (C) at
    (\ysc,-2*\factdiff*\xscb) {\fmt{C}};
    \node[draw, rectangle, fill=caribbeangreen!80] (a) at
    (0,3*\factdiff*\xsc) {\fmt{onTable(B)}};
    \node[draw, rectangle, fill=caribbeangreen!80] (b) at
    (0,1.5*\factdiff*\xsc) {\fmt{on(C,B)}};
    \node[draw, rectangle, fill=yellow!80] (c) at
    (0,0*\factdiff*\xsc) {\fmt{on(B,A)}};
    \node[draw, rectangle, fill=yellow!80] (d) at
    (0,-1.5*\factdiff*\xsc) {\fmt{onTable(C)}};
    \node[draw, rounded corners=1.5, fill=red!30] (e) at
    (0,-3*\factdiff*\xsc) {\fmt{onTable(A)}};
}
\newcommand{\bwXshift}{0cm}
\newcommand{\bwYshift}{-0.4cm}
\newcommand{\blocksfontsize}{\scriptsize}
\newcommand{\blocksize}{0.25}
\newcommand{\graphScale}{0.375}
\newcommand{\vecSize}{0.25}
\newcommand{\midShift}{0.25cm}
\begin{tikzpicture}
    \tikzset{
        outer sep=0pt,
        model/.style={
                draw=\trainColour,
                rectangle,
                minimum width=1.5cm,
                minimum height=0.5cm,
                font=\scriptsize,
            },
        heuristic/.style={
                font=\scriptsize,
            },
        vectorSquare/.style={
                draw,
                rectangle,
                minimum width=\vecSize,
                minimum height=\vecSize,
                text width=\vecSize,
                text height=\vecSize,
                text depth=0,
            },
    }
    \node[draw=black, rectangle, rounded corners, minimum width=\segmentWidth, minimum height=\segmentHeight]
    (blocks) at (\bwXshift+0.625cm,0) {};
    \begin{scope}[xshift=\bwXshift, yshift= 2*\blocksize cm+\bwYshift]
        \node at (0, 3.75*\blocksize) {\tiny$\problem$};
        \draw (1*\blocksize,0) rectangle (2*\blocksize,1*\blocksize) node[midway] {\blocksfontsize$\texttt{A}$};
        \draw (3*\blocksize,0) rectangle (4*\blocksize,1*\blocksize) node[midway] {\blocksfontsize$\texttt{B}$};
        \draw (3*\blocksize,1*\blocksize) rectangle (4*\blocksize,2*\blocksize) node[midway] {\blocksfontsize$\texttt{C}$};
        \draw (2.5*\blocksize,3*\blocksize) node[align=center] {\tiny initial state};

        \draw(-0.25*\blocksize, 0) rectangle (5.25*\blocksize, -0.125*\blocksize) node[midway,below=0.1*\blocksize] {};
    \end{scope}
    \begin{scope}[xshift=\bwXshift, yshift=-2*\blocksize cm+\bwYshift]
        \draw (1*\blocksize,0*\blocksize) rectangle (2*\blocksize,1*\blocksize) node[midway] {\blocksfontsize$\texttt{A}$};
        \draw (1*\blocksize,1*\blocksize) rectangle (2*\blocksize,2*\blocksize) node[midway] {\blocksfontsize$\texttt{B}$};
        \draw (3*\blocksize,0*\blocksize) rectangle (4*\blocksize,1*\blocksize) node[midway] {\blocksfontsize$\texttt{C}$};
        \draw (2.5*\blocksize,3*\blocksize) node[align=center] {\tiny goal condition};

        \draw(-0.25*\blocksize, 0) rectangle (5.25*\blocksize, -0.125*\blocksize) node[midway,below=0.1*\blocksize] {};
    \end{scope}
    \node[draw=black, rectangle, rounded corners, minimum width=\segmentWidth, minimum height=\segmentHeight]
    (graph) at (\segmentShift + 0.425cm,0) {};
    \begin{scope}[yshift=-0.1cm, scale=\graphScale, every node/.append style={transform shape}, xshift=1/\graphScale*\segmentShift*1.05]
        \nilgTikz
        \node at (-1.05, 3) {\huge$\graph$};
    \end{scope}
    \node[draw=black, rectangle, rounded corners, minimum width=\segmentWidth, minimum height=\segmentHeight]
    (vector) at (2*\segmentShift + 0.12cm + \midShift, 0) {};

    \def\limer{0.05}
    \begin{scope}[yshift=0, xshift=2.025*\segmentShift + \midShift]
        \node at (-0.55,1) {\scriptsize$\multiset$};
        \begin{scope}[yshift=-0.5]
            \node at (\limer,3*\vecSize) {$(1, 4)$};
            \node at (\limer,1*\vecSize) {$(2, 7)$};
            \node at (\limer,-1*\vecSize) {$(3, 6)$};
            \node at (\limer,-3*\vecSize) {$(4, 2)$};
            \node at (1.5*\limer,-4.5*\vecSize) {\ldots};
        \end{scope}
    \end{scope}

    \newcommand{\aboveShift}{0}
    \draw[->,black] (blocks) -- (graph)     node[midway,above=\aboveShift,align=center,font=\scriptsize] {Graph\\Representation};
    \draw[->,black] (graph) -- (vector)     node[midway,above=\aboveShift,align=center,font=\scriptsize] {\phantom{p}Weisfeiler-Leman\phantom{p}\\\phantom{p}Algorithm\phantom{p}};

\end{tikzpicture}
  \caption{The pipeline for generating WL Features for a Blocksworld problem (\texttt{clear} atoms omitted).}
  \label{fig:wlplan}
\end{figure}

\begin{algorithm}[t]
  \caption{The Weisfeiler-Leman Algorithm}\label{alg:wl}
  \KwData{A graph $\graph = \gen{\nodes, \edges, \featCat, \featEdge}$, injective $\hash$ function, and number of iterations $L$.}  
\KwResult{Multiset of colours $\multiset$.}  
$c^{0}(v) \la \featCat(v), \forall v \in \nodes$ \label{line:wl:init}\\
\For{$l=1,\ldots,L$ \normalfont{\textbf{do for}} $v \in \nodes$}{ \label{line:wl:update1}
    $c^{l}(v) \la 
    \hash
    \lr{c^{l-1}(v), 
    \bigcup_{\iota\in\edgeCat}\mseta{(c^{l-1}(u), \iota) \mid u \in \neighbour_{\iota}(v)}}
    $ \label{line:wl:update2}
} 
\Return{$\bigcup_{l=0,\ldots,L}\mseta{c^{l}(v) \mid v \in \nodes}$} \label{line:wl:return}

\end{algorithm}

\paragraph{Weisfeiler-Leman Features for Planning}\label{ssec:graph}
Weisfeiler-Leman Features (WLFs) are introduced as state-centric feature generators for planning problems for use with learning heuristics for search~\cite{chen.etal.icaps2024}.
WLFs have been extended to handle numeric planning problems~\cite{chen.thiebaux.neurips2024}, probabilistic planning problems~\cite{zhang.2024}, and for learning action set heuristics~\cite{wang.trevizan.icaps2025}.
As summarised in \Cref{fig:wlplan}, WLFs are generated via a 2 step process, consisting of (1) transforming a planning problem $\problem$ into a graph with edge labels\footnote{Also known as `relational structures' in other communities.} $\graph$, and (2) running the Weisfeiler-Leman (WL) algorithm on the graph to generate a set of features $\multiset$ that takes the form of a set of tuples $(f_i, v_i) \in \N \times \N$.

\def\ob{\texttt{ob}}
\def\ag{\texttt{ag}}
\def\ug{\texttt{ug}}
\def\ap{\texttt{ap}}

Step (1) may use any graph representation of the planning problem.
Formally, we denote a graph with categorical node features and edge labels by a tuple $\graph = \gen{\nodes, \edges, \featCat, \featEdge}$ where $\nodes$ is a set of nodes, $\edges \subseteq \nodes \times \nodes$ a set of edges, $\featCat:\nodes \to \nodeCat$ the discrete node features, and $\featEdge:\edges \to \edgeCat$ the edge labels, where $\nodeCat$ and $\edgeCat$ are finite sets of symbols.
The neighbourhood of a node $u \in \nodes$ with respect to an edge label $\iota$ is defined by $\neighbour_{\iota}(u) = \set{v \in \nodes \mid e=\gen{u,v} \in \edges \land \featEdge(e) = \iota}$.
Let $\graphRep: \seta{\problem} \to \seta{\graph}$ denote a \emph{graph representation function} that maps planning problems into graphs.

For our experiments, we use the \emph{Instance Learning Graph} (ILG)~\cite{chen.etal.icaps2024} representation function.
Given a problem $\problem=\gen{\objects, \predicates, \schemata, s_0, \goal}$ we define $\text{ILG}(\problem) = \gen{\nodes, \edges, \featCat, \featEdge}$ with
$\nodes = \objects \cup s_0 \cup \goal$, 
$\edges = \bigcup_{p=P(\objs)\in s_0 \cup \goal} \seta{\!\gen{p, \objs_1}\!,\!\gen{\objs_1, p}\!, \ldots, \!\gen{p_{\arity(P)}, \objs}\!,\!\gen{\objs_{\arity(P)}, p}\!}$, 
$\featCat:\nodes \to (\set{\ap,\ug,\ag} \times \predicates) \cup \set{\ob}$ defined by
$u \mapsto \ob$ if $u \in \objects$, 
$u \mapsto (\ag,P)$ if $u=P(\objs)\in s_0 \cap \goal$,
$u \mapsto (\ap,P)$ if $u=P(\objs)\in s_0 \setminus \goal$, and
$u \mapsto (\ug,P)$ if $u=P(\objs)\in \goal \setminus s_0$, and
$\featEdge:\edges \to \N$ defined by $\gen{p, \objs_i} \mapsto i$.

Step (2) transforms the graph into a multiset $\multiset = \set{(f_i, v_i) \in \N \times \N \mid i \in \range{n}}$ that consists of a set of elements $f_i$ and their counts $v_i$.
The transformation is performed by the WL algorithm described in \Cref{alg:wl} where the input is a graph $\graph$ and a hyperparameter $L \in \N$.
The underlying concept of the WL algorithm iteratively updates node colours based on the colours of their
neighbours.
\Cref{line:wl:init} initialises node graph colours as their categorical node features.
\Crefrange{line:wl:update1}{line:wl:update2} iteratively update the colour of each node $v$ in the graph by collecting all its neighbours and the corresponding edge label $(u, \iota)$ into a multiset.
This multiset is then hashed alongside $v$'s current colour with an injective function to produce a new refined colour.
In practice, the injective hash function is built lazily, where every time a new multiset is encountered, it is mapped to a new, unseen hash value.
After $L$ iterations, the multiset of all node colours seen throughout the algorithm is returned.
Given a graph representation $\graphRep$, we denote the set of WLF features of a given problem by
\begin{align}
  \WL^{\graphRep}(\problem) = \multiset. \label{eqn:wl}
\end{align}

\paragraph{Quantified-Both Novelty Heuristics}
A heuristic (function) is a function\footnote{Heuristic functions usually take nonnegative values. In the context of satisficing planning and greedy best first search, it is fine to exhibit negative heuristic values.} $h: S \to \R$, where lower values are favoured for search.
The \emph{Quantified Both} (QB) novelty heuristic~\cite{katz.etal.icaps2017} creates a new heuristic $h_{\qb}$ from an existing heuristic $h$ that favours novel states and distinguishes both novel and non-novel states using $h$.
Given a problem and a set of states seen during search $C$, the QB heuristic is defined by
\begin{align}
  h_{\qb}(s) &=
  \begin{cases}
    - \abs{\new(s)}, & \text{if $\abs{\new(s)} > 0$,} \\
    + \abs{\old(s)}, & \text{otherwise,}              \\
  \end{cases}
  \label{eqn:qb} \\
  \new(s) & = \seta{p \in s \mid h(s) < \min_{t \in C, p \in t} h(t)} \label{eqn:new}  \\
  \old(s) & = \seta{p \in s \mid h(s) > \min_{t \in C, p \in t} h(t)}. \label{eqn:old}
\end{align}

As discussed by~\citet[page 7]{katz.etal.icaps2017}, the QB novelty heuristic differs from the Partition Novelty (PN) measure heuristic~\cite{lipovetzky.geffner.aaai2017,correa.seipp.icaps2022} in that PN does not favour states with lower $h$ values, and does not try to distinguish states that are not novel.
For example, two states with original $h$ values of 9 and 7 that have the same PN values are distinguished by QB.

\section{Symmetry-Invariant\\Weisfeiler-Leman Novelty Heuristics}
One can generalise QB heuristics to take arbitrary features in place of atoms in the definitions of $\new$ and $\old$ in \Crefrange{eqn:new}{eqn:old}.
Specifically, let $\features: S \to 2^{\Sigma_\features}$ denote a \emph{feature mapping} of states to sets of features in $\Sigma_\features$.
An example is the feature $\AT: S \to S$ that maps a state to itself, the set of true atoms under the closed world assumption.

Then we define the \emph{Generalised Quantified Both} novelty heuristic that takes as input a heuristic $h$, feature mapping $\features$ of states, and a set of states seen during search $C$ by
\begin{align}
  h_{\qb}^\features(s) &=
  \begin{cases}
    - \abs{\features_{\neww}(s)}, & \text{if $\abs{\features_{\neww}(s)} > 0$,} \\
    + \abs{\features_{\oldd}(s)}, & \text{otherwise,}                           \\
  \end{cases}
  \label{eqn:genqb} \\
  \features_{\neww}(s) & = \seta{p \in \features(s) \mid h(s) < \min_{t \in C, p \in \features(t)} h(t)} \label{eqn:gnew}  \\
  \features_{\oldd}(s) & = \seta{p \in \features(s) \mid h(s) > \min_{t \in C, p \in \features(t)} h(t)}. \label{eqn:gold}
\end{align}

Note that the original QB heuristic in \Cref{eqn:qb} is subsumed by the Generalised QB heuristic as $h_{\qb} = h_{\qb}^{\AT}$.
We will refer to $h_{\qb}^{\AT}$ as the \emph{Atom Novelty Heuristic}.
Similarly, we can define a QB heuristic using WL features from \Cref{eqn:wl} denoted by $h_{\qb}^{\WL^{\graphRep}}$.
We name $h_{\qb}^{\WL^{\graphRep}}$ as the \emph{Weisfeiler-Leman Novelty heuristic}.
Note also that feature mappings can be combined to make new novelty heuristics, such as $h_{\qb}^{\AT;\WL^{\graphRep}}$ that uses the feature mapping $\AT;\WL^{\graphRep}: S \to S \cup \Sigma_{\WL^{\graphRep}}$ defined by $s \mapsto s \cup \WL^{\graphRep}(\problem_s)$.

We now sketch a proof of how WL novelty heuristics are symmetry invariant.
Following~\citet[Definition 7]{drexler.etal.kr2024} and \citet[Definition 16]{chen.etal.tkr2025}, we define an equivalence relation on states of a planning problem via bijection between objects, in contrast to work by \citet{sievers.etal.icaps2019} that reduce problems to graph automorphisms.
\begin{definition}[Equivalence Relation]\label{defn:equiv}
  Let $\problem$ be a planning problem.
  We define a relation $\symm$ on states in $\problem$ by $s_1 \symm s_2$ if there exists a permutation (bijective mapping from a set to itself) $f: \problem[\objects] \to \problem[\objects]$ such that $s_2 = \set{P(f(o_1), \ldots, f(o_n)) \mid P(o_1, \ldots, o_n) \in s_1}$.
\end{definition}

Next, following~\citet[Theorem 12]{drexler.etal.kr2024}, we define a notion of symmetry-invariant graph representation for states in a planning problem.
\begin{definition}[Symmetry-Invariant Graph Representation]\label{definition:sym}
  Let $\graphRep: \seta{\problem} \to \seta{\graph}$ be a graph representation function.
  We say that $\graphRep$ is a \emph{symmetry-invariant graph representation function} if for any given problem $\problem$, for any two states $s, s'$ in $\problem$, we have that $s \symm s'$ if and only if $\graphRep(\problem_s)$ is (graph) isomorphic to $\graphRep(\problem_{s'})$.
\end{definition}

Now, we have a known fact that the WL algorithm is a graph invariant, meaning that the features output by \Cref{alg:wl} are the same for isomorphic graphs.
\begin{lemma}[\citet{weisfeiler.leman.ni1968}]\label{lem:wl}
  The WL algorithm is a graph invariant; i.e., if two graphs $\graph_1$ and $\graph_2$ are isomorphic, then $\WL(\graph_1) = \WL(\graph_2)$.
\end{lemma}

We now state our main theoretical result following from the previous lemma that states that WLF Novelty heuristics are symmetry-invariant heuristics.
\begin{proposition}[WL Novelty Heuristics are symmetry-invariant]
  Let $\graphRep$ be a symmetry-invariant graph representation function.
  Then $h^{\WL^{\graphRep}}_{\qb}$ is symmetry-invariant; i.e. $s \symm s'$ implies $h^{\WL^{\graphRep}}_{\qb}(s) = h^{\WL^{\graphRep}}_{\qb}(s')$.
\end{proposition}
\begin{proof}[Proof Sketch]
    Let $s \symm s'$. Then $\graphRep(\problem_s)$ is isomorphic to $\graphRep(\problem_{s'})$ by the proposition assumption and \Cref{definition:sym}.
    Then by \Cref{lem:wl}, $\WL(\graphRep(\problem_s)) = \WL(\graphRep(\problem_{s'}))$.
    Then \Cref{eqn:gnew} is the same for $s$ and $s'$, and similarly for \Cref{eqn:gold} where $\features(t) := \WL(\graphRep(P_t))$ for all states $t$.
\end{proof}

We conclude this section with an example sketch of how the original (quantified both) novelty heuristics are not symmetry-invariant.
\begin{example}[Atom Novelty Heuristics are not symmetry-invariant]
  Let us consider the canonical Childsnack domain~\cite{fuentetaja.rosa.ac2016} where one must make sandwiches for children.
  Some children have hard constraints on what sandwiches they can eat, e.g. some are allergic to gluten.
  Gluten-free ingredients can be with non-gluten-free ingredients to make sandwiches that are not safe for children allergic to gluten.
  Delete-free heuristics such as \hff{} perform badly as it thinks that making one gluten-free sandwich is sufficient for feeding all children and has to explore a large symmetric state-space to clear out deadends which occur when non-gluten-free sandwiches are made with gluten-free ingredients.
  For example consider the states $\{\texttt{exists(sw1, gluten-free)}$, $\texttt{exists(sw2, gluten)}$, $\texttt{exists(sw3, gluten)}\}$ and $\{\texttt{exists(sw4, gluten-free)}$, $\texttt{exists(sw5, gluten)}$, $\texttt{exists(sw6, gluten)}\}$ which are equivalent under the relation $\symm$ in \Cref{defn:equiv} if the goals are of the form $\bigwedge_{\texttt{c} \in \text{Children}} \texttt{served(c)}$.
  However, atom novelty heuristics will see both states as novel if they have not encountered any of the facts before.
\end{example}

\section{Experiments}
We run experiments to help answer the following questions.
Our preliminary results are summarised as answers in pink, but we refer to later in the text for a more detailed analysis.

\newcommand{\answer}[1]{{\color{magenta}{#1}}}
\begin{itemize}[leftmargin=*]
  \item When does novelty help improve the base heuristic? \answer{When the base heuristic is informative so exploration helps.}
  \item Are WL features useful for generating novelty heuristics? \answer{Yes in general, but not always.}
  \item Is combining feature sets helpful for novelty heuristics? \answer{Yes in general, but not always.}
\end{itemize}

\begin{table*}[t]
    \newcommand{\midsplitterA}{
\cmidrule(l){2-5}
\cmidrule(l){6-9}
\cmidrule(l){10-13}
}
\newcommand{\midsplitterB}{
\cmidrule{1-1}
\cmidrule(l){2-5}
\cmidrule(l){6-9}
\cmidrule(l){10-13}
}
\begin{tabularx}{\textwidth}{l *{12}{Y}}
    \toprule
    & \multicolumn{4}{c}{\hgc (least informative)}
    & \multicolumn{4}{c}{\hadd (informative)}
    & \multicolumn{4}{c}{\hff (most informative)}
    \\
    \midsplitterA
    $\Sigma$ Coverage
    & Base & \atoms & \wl & \atomswl
    & Base & \atoms & \wl & \atomswl
    & Base & \atoms & \wl & \atomswl
    \\
    \midsplitterB
Unnormalised & \textbf{1959} & \negg{1892} & \negg{1858} & \negg{1824} & 1868 & \poss{1992} & \poss{2035} & \bestbest{\textbf{{2096}}} & 1800 & \poss{1921} & \poss{1977} & \textbf{\poss{2094}} \\
Normalised & \textbf{38.73} & \negg{37.07} & \negg{37.78} & \negg{35.87} & 40.57 & \poss{43.00} & \poss{42.90} & \textbf{\poss{45.91}} & 40.25 & \poss{42.51} & \poss{40.61} & \bestbest{\textbf{{45.97}}} \\
    \bottomrule
\end{tabularx}

    \caption{Total unnormalised and normalised coverage ($\uparrow$) of heuristics for single-queue GBFS across different domains. Green/red cells indicate that a novelty heuristic (\atoms, \wl, \atomswl) performs better/worse than their base heuristic (\hgc, \hadd, \hff). Blue cells indicate the best value per row. Bold font indicates the best score in each local heuristic group.}
    \label{tab:summary}
\end{table*}

\paragraph{Setup}
We experiment with two benchmark suites: the set of classical planning problems from the 1998 to 2023 International Planning Competitions (IPC), and the Hard To Ground (HTG)~\cite{correa.etal.icaps2020,lauer.etal.ijcai2021}.
The latter benchmark is chosen to demonstrate that our approach is easily applicable to the lifted setting.
For the IPC (resp. HTG) domains, we use heuristics and implement our approaches in Fast Downward Version 24.06\footnote{\url{https://github.com/aibasel/downward}; release 24.06} (resp. Powerlifted 2024\footnote{\url{https://github.com/abcorrea/powerlifted}; commit 736b0c}).
We experiment with the goalcount (\hgc), additive~\cite{bonet.geffner.ai2001,correa.etal.icaps2021} (\hadd), and FF~\cite{hoffmann.nebel.jair2001,correa.etal.aaai2022} (\hff) heuristics.
We consider their extensions under the Generalised Quantified Both framework described in \Cref{eqn:genqb} with the feature sets consisting of mapping states to atoms (\atoms), to WL features described in \Cref{alg:wl} with $L=2$ and the ILG as the underlying graph representation (\wl), and the union of both feature sets (\atomswl).
All heuristics are used in single-queue greedy best first search (GBFS) in their respective planning engines with a 4GB memory limit and 300 second timeout.
We refer to \Cref{tab:summary} for total unnormalised and normalised\footnote{This refers to the coverage divided by the total number of problems in a domain solved by at least one configuration. Normalised coverage accounts for heavily skewed problem distributions, e.g. IPC Miconic and HTG Genome each have over 300 problems each.} coverage scores, and \Cref{sec:coverage} for coverage scores per domain.

\paragraph{When does novelty help improve the base heuristic?}
From \Cref{tab:summary}, we notice an interesting trend: novelty heuristics degrade the performance of \hgc{}, but improve the \hadd{} and \hff{} heuristics.
Both the normalised and unnormalised coverage of \hgc{} is higher than its novelty counterparts, and the performance degrades with novelty across domains more often than it improves.
This can be explained in the exploration vs. exploitation viewpoint of search: novelty heuristics boost exploration, while the informativeness of a heuristic corresponds to exploitation.
The \hgc{} heuristic is not very informative and thus adding novelty on top of it leads to more exploration than exploitation.
Conversely, the \hadd{} and \hff{} heuristics are more informative heuristics in ascending order and result in greater coverage gains when used with novelty heuristics.
Interestingly, we observe from \Cref{sec:expansions} that novelty heuristics \atoms{} and \wl{} generally expand \emph{more} nodes than their base heuristic on problems that are solved by both heuristics.
This is reasonable as exploration from novelty heuristics is forcing more expansions that may not be necessary when the base heuristic is strong enough. 

\paragraph{Are WL features useful for generating novelty heuristics?}
From the previous answer, we note that WL features generally improve upon the base heuristic if the base heuristic is sufficiently informative.
Thus, we analyse how they compare to the original quantified both novelty heuristics.
We observe from \Cref{tab:summary} that \wl{} novelty heuristics have a lower normalised coverage than \atoms{} novelty heuristics, but higher unnormalised coverage.
This suggests that they are generally incomparable.
However, when looking at select domains, as from the theoretical intuition, \wl{} novelty heuristics perform better on domains where a heuristic gets stuck on symmetric states such as in Childsnack.
Conversely, the WL algorithm is an incomplete graph isomorphism problem which means that it may mistakingly mark some asymmetric states as symmetric.
This can have an adverse affect on exploration on domains where WL cannot distinguish asymmetric states which require to be expanded.

\paragraph{Is combining feature sets helpful for novelty heuristics?}
From the aforementioned tables, we notice that combining both \atoms{} and \wl{} features into a \emph{single} novelty heuristic (\atomswl{}) achieves the best scores overall for \hadd{} and \hff{}, and by a non-trivial margin (over 12\% and 16\% increase, respectively).
This may be attributed to \atoms{} and \wl{} having diverse and contrasting features that are complementary to each other that can further guide exploration.
Conversely, \atomswl{} performs worse when using the least informative \hgc{} heuristic as explained in the first question.

\section{Related Work, Discussion, and Conclusion}
Symmetries are an extensively studied topic in planning.
Most closely related to our work, \citet{shleyfman.etal.aaai2015} study whether existing domain-independent heuristics are symmetry invariant under a notion of symmetry that considers actions and action costs.
On the other hand, novelty heuristics are transition agnostic which simplifies our analysis of symmetries.
Orthogonally, symmetries have been used for state-space pruning in various planning settings and search algorithms~\cite{pochter.etal.aaai2011,domshlak.etal.icaps2012,domshlak.etal.icaps2013,wehrle.etal.ijcai2015,gnad.etal.icaps2017,shleyfman.etal.icaps2023,bai.etal.icaps2025}.
\citet{drexler.etal.kr2024} leverage symmetries to reduce redundant computation in generating training data for generalised planners.

Existing novelty heuristics~\cite{lipovetzky.geffner.aaai2017,katz.etal.icaps2017,ramirez.etal.aamas2018,teichteilkoenigsbuch.etal.ijcai2020,chen.thiebaux.socs2024} all derive their features from atoms and numeric variable assignments.
Our work differs in generalising the novelty heuristic framework that takes as input a feature generator on top of a base heuristic, as described in \Cref{eqn:genqb}.
This simple insight allows us to make use of other existing work on generating features for planning problems such as description logic features~\cite{martin.geffner.ai2004} and embeddings from pretrained, domain-independent neural networks~\cite{shen.etal.icaps2020,chen.etal.aaai2024} as examples.

Our approach is simple yet opens up many new avenues of research.
On the empirical side, we can further boost planning performance.
This involves experimenting with additional WL parameters~\cite{chen.ecai2025} and other feature generators, as well as orthogonal search techniques such as using lazy heuristic evaluation, preferred operators, and multiple queues.
On the theoretical side, one can extend and generalise the complexity theory of width, serialisation and subgoaling~\cite{lipovetzky.geffner.ecai2012,dold.helmert.aaai2024,drexler.etal.jair2024} for planning.

{
  \small
  \bibliography{wlf_novelty.bib}
}

\ifshowappendix
\clearpage
\appendix
\onecolumn

\section{Coverage Results}\label{sec:coverage}

\begin{table*}[ht!]
  \centering
  \tabcolsep 1pt
  \tiny

\newcommand{\midsplitterA}{
\cmidrule(l){2-5}
\cmidrule(l){6-9}
\cmidrule(l){10-13}
}
\newcommand{\midsplitterB}{
\cmidrule{1-1}
\cmidrule(l){2-5}
\cmidrule(l){6-9}
\cmidrule(l){10-13}
}
\begin{tabularx}{\textwidth}{X *{12}{Y}}
    \toprule
    & \multicolumn{4}{c}{\hgc (least informative)}
    & \multicolumn{4}{c}{\hadd (informative)}
    & \multicolumn{4}{c}{\hff (most informative)}
    \\
    \midsplitterA
    Domain 
    & Base & \atoms & \wl & \atomswl
    & Base & \atoms & \wl & \atomswl
    & Base & \atoms & \wl & \atomswl
    \\
    \midsplitterB
    agricola & 0 & \textbf{\poss{3}} & 0 & 0 & \bestbest{\textbf{11}} & \negg{4} & \negg{0} & \negg{3} & \textbf{6} & \negg{5} & \negg{0} & \negg{3} \\
assembly & 0 & 0 & \textbf{\poss{10}} & \poss{7} & 11 & \bestbest{\textbf{{30}}} & \poss{27} & \poss{23} & \bestbest{\textbf{30}} & \bestbest{\textbf{30}} & \negg{25} & \negg{27} \\
barman & 4 & \negg{0} & \textbf{\poss{5}} & \negg{0} & 0 & 0 & \textbf{\poss{1}} & 0 & 2 & \bestbest{\textbf{{11}}} & \poss{4} & \poss{4} \\
blocks & \bestbest{\textbf{35}} & \negg{33} & \negg{34} & \bestbest{\textbf{35}} & \bestbest{\textbf{35}} & \bestbest{\textbf{35}} & \bestbest{\textbf{35}} & \bestbest{\textbf{35}} & \bestbest{\textbf{35}} & \bestbest{\textbf{35}} & \bestbest{\textbf{35}} & \bestbest{\textbf{35}} \\
caldera & \textbf{16} & \negg{13} & \negg{14} & \negg{10} & \bestbest{\textbf{18}} & \bestbest{\textbf{18}} & \negg{17} & \negg{16} & 12 & \poss{16} & 12 & \textbf{\poss{17}} \\
cavediving & \bestbest{\textbf{7}} & \bestbest{\textbf{7}} & \bestbest{\textbf{7}} & \bestbest{\textbf{7}} & \bestbest{\textbf{7}} & \negg{6} & \bestbest{\textbf{7}} & \bestbest{\textbf{7}} & \bestbest{\textbf{7}} & \bestbest{\textbf{7}} & \bestbest{\textbf{7}} & \bestbest{\textbf{7}} \\
childsnack & 0 & 0 & 0 & 0 & 2 & \negg{1} & \negg{1} & \bestbest{\textbf{{4}}} & 0 & 0 & \poss{1} & \textbf{\poss{3}} \\
citycar & 0 & 0 & \poss{5} & \textbf{\poss{8}} & \bestbest{\textbf{20}} & \negg{15} & \bestbest{\textbf{20}} & \bestbest{\textbf{20}} & 0 & \poss{8} & \textbf{\poss{19}} & \textbf{\poss{19}} \\
data & 1 & 1 & \bestbest{\textbf{{4}}} & \poss{3} & 1 & \textbf{\poss{3}} & \textbf{\poss{3}} & \textbf{\poss{3}} & \bestbest{\textbf{4}} & \negg{3} & \negg{2} & \bestbest{\textbf{4}} \\
depot & 14 & \poss{16} & \negg{13} & \textbf{\poss{17}} & 13 & \poss{19} & \bestbest{\textbf{{21}}} & \bestbest{\textbf{{21}}} & 15 & \poss{17} & \poss{19} & \bestbest{\textbf{{21}}} \\
driverlog & \textbf{19} & \textbf{19} & \negg{18} & \textbf{19} & 18 & \textbf{\poss{19}} & \negg{17} & \textbf{\poss{19}} & 17 & \poss{19} & \poss{18} & \bestbest{\textbf{{20}}} \\
elevators & 33 & \bestbest{\textbf{{39}}} & \negg{24} & \negg{29} & 15 & \textbf{\poss{30}} & \textbf{\poss{30}} & \textbf{\poss{30}} & 11 & \negg{8} & \poss{16} & \textbf{\poss{19}} \\
floortile & 0 & 0 & \textbf{\poss{4}} & \poss{3} & \bestbest{\textbf{10}} & \negg{9} & \negg{8} & \negg{8} & \bestbest{\textbf{10}} & \negg{6} & \bestbest{\textbf{10}} & \bestbest{\textbf{10}} \\
folding & 8 & \bestbest{\textbf{{9}}} & \negg{6} & \negg{6} & 0 & 0 & \textbf{\poss{3}} & \textbf{\poss{3}} & 0 & \poss{2} & \textbf{\poss{3}} & \textbf{\poss{3}} \\
freecell & 46 & \textbf{\poss{57}} & \negg{43} & \poss{53} & \bestbest{\textbf{80}} & \negg{77} & \negg{75} & \negg{79} & \textbf{79} & \textbf{79} & \negg{76} & \negg{77} \\
ged & \bestbest{\textbf{20}} & \bestbest{\textbf{20}} & \negg{19} & \negg{19} & 0 & \textbf{\poss{11}} & \poss{8} & \poss{9} & 0 & 0 & \textbf{\poss{3}} & \poss{1} \\
grid & \textbf{3} & \textbf{3} & \textbf{3} & \textbf{3} & 3 & \bestbest{\textbf{{5}}} & \negg{2} & \bestbest{\textbf{{5}}} & 4 & \bestbest{\textbf{{5}}} & \negg{2} & \bestbest{\textbf{{5}}} \\
gripper & \bestbest{\textbf{20}} & \bestbest{\textbf{20}} & \bestbest{\textbf{20}} & \bestbest{\textbf{20}} & \bestbest{\textbf{20}} & \bestbest{\textbf{20}} & \bestbest{\textbf{20}} & \bestbest{\textbf{20}} & \bestbest{\textbf{20}} & \bestbest{\textbf{20}} & \bestbest{\textbf{20}} & \bestbest{\textbf{20}} \\
hiking & \textbf{2} & \textbf{2} & \textbf{2} & \textbf{2} & 19 & \negg{18} & \bestbest{\textbf{{20}}} & 19 & \bestbest{\textbf{20}} & \bestbest{\textbf{20}} & \bestbest{\textbf{20}} & \bestbest{\textbf{20}} \\
labyrinth & \bestbest{\textbf{4}} & \negg{2} & \negg{3} & \negg{2} & 0 & 0 & 0 & 0 & 0 & 0 & 0 & 0 \\
logistics & 35 & \textbf{\poss{36}} & \negg{34} & \textbf{\poss{36}} & \textbf{53} & \negg{48} & \negg{51} & \negg{52} & \bestbest{\textbf{55}} & \negg{50} & \negg{53} & \negg{52} \\
maintenance & 14 & \textbf{\poss{16}} & \negg{7} & \negg{9} & 16 & \bestbest{\textbf{{17}}} & \negg{10} & \negg{13} & \textbf{11} & \negg{7} & \negg{6} & \negg{6} \\
miconic & \textbf{371} & \textbf{371} & \negg{364} & \negg{363} & \bestbest{\textbf{439}} & \negg{436} & \negg{427} & \bestbest{\textbf{439}} & 437 & \bestbest{\textbf{{439}}} & \bestbest{\textbf{{439}}} & \negg{436} \\
movie & \bestbest{\textbf{30}} & \bestbest{\textbf{30}} & \bestbest{\textbf{30}} & \bestbest{\textbf{30}} & \bestbest{\textbf{30}} & \bestbest{\textbf{30}} & \bestbest{\textbf{30}} & \bestbest{\textbf{30}} & \bestbest{\textbf{30}} & \bestbest{\textbf{30}} & \bestbest{\textbf{30}} & \bestbest{\textbf{30}} \\
mprime & 21 & \textbf{\poss{22}} & 21 & 21 & 31 & \bestbest{\textbf{{35}}} & \bestbest{\textbf{{35}}} & \bestbest{\textbf{{35}}} & 30 & \bestbest{\textbf{{35}}} & \bestbest{\textbf{{35}}} & \bestbest{\textbf{{35}}} \\
mystery & 15 & 15 & \negg{14} & \textbf{\poss{16}} & 18 & \bestbest{\textbf{{19}}} & \bestbest{\textbf{{19}}} & \bestbest{\textbf{{19}}} & 17 & \bestbest{\textbf{{19}}} & \bestbest{\textbf{{19}}} & \bestbest{\textbf{{19}}} \\
nomystery & 6 & \textbf{\poss{10}} & \poss{9} & \poss{8} & 6 & \textbf{\poss{12}} & 6 & \poss{11} & 9 & \bestbest{\textbf{{17}}} & \negg{8} & \poss{15} \\
nurikabe & \bestbest{\textbf{13}} & \bestbest{\textbf{13}} & \negg{12} & \negg{12} & 9 & \negg{7} & \negg{8} & \textbf{\poss{10}} & 7 & \textbf{\poss{11}} & 7 & \poss{9} \\
openstacks & \bestbest{\textbf{40}} & \negg{34} & \negg{33} & \negg{28} & 29 & \negg{28} & \textbf{\poss{34}} & \textbf{\poss{34}} & 36 & 36 & \textbf{\poss{38}} & \textbf{\poss{38}} \\
optical & 5 & \negg{4} & \textbf{\poss{18}} & \textbf{\poss{18}} & 6 & \bestbest{\textbf{{21}}} & \poss{16} & \poss{15} & 4 & \textbf{\poss{13}} & \poss{10} & \poss{7} \\
parking & 0 & 0 & 0 & 0 & 10 & \poss{16} & \poss{21} & \textbf{\poss{23}} & 23 & \negg{5} & \poss{29} & \bestbest{\textbf{{31}}} \\
pegsol & \bestbest{\textbf{50}} & \negg{46} & \negg{44} & \negg{44} & \bestbest{\textbf{50}} & \negg{44} & \negg{46} & \negg{48} & \bestbest{\textbf{50}} & \negg{46} & \negg{48} & \negg{48} \\
philosophers & 20 & \textbf{\poss{26}} & \negg{11} & \negg{12} & \textbf{47} & \negg{28} & \negg{24} & \negg{22} & \bestbest{\textbf{48}} & \bestbest{\textbf{48}} & \negg{17} & \negg{24} \\
pipesworld & 61 & \textbf{\poss{65}} & \negg{52} & \negg{60} & 43 & \poss{80} & \poss{76} & \bestbest{\textbf{{84}}} & 50 & \poss{81} & \poss{74} & \bestbest{\textbf{{84}}} \\
psr & \bestbest{\textbf{92}} & \negg{79} & \negg{69} & \negg{69} & \textbf{82} & \negg{78} & \negg{69} & \negg{69} & 58 & \negg{56} & \textbf{\poss{73}} & \textbf{\poss{73}} \\
recharging & \bestbest{\textbf{12}} & \negg{11} & \negg{8} & \negg{7} & 9 & \textbf{\poss{10}} & \negg{8} & \negg{7} & 8 & \textbf{\poss{11}} & 8 & \negg{7} \\
ricochet & \textbf{4} & \negg{0} & \negg{2} & \negg{0} & \textbf{4} & \negg{3} & \negg{0} & \negg{1} & \bestbest{\textbf{12}} & \negg{5} & \negg{2} & \negg{1} \\
rovers & \textbf{21} & \textbf{21} & \negg{17} & \negg{19} & 26 & \textbf{\poss{30}} & 26 & \poss{28} & 24 & \bestbest{\textbf{{31}}} & \poss{30} & \poss{29} \\
rubiks & 4 & \textbf{\poss{6}} & 4 & 4 & 4 & \textbf{\poss{8}} & \poss{5} & \poss{5} & \bestbest{\textbf{20}} & \negg{9} & \negg{6} & \negg{6} \\
satellite & 14 & \textbf{\poss{18}} & \negg{13} & 14 & \bestbest{\textbf{30}} & \negg{26} & \negg{28} & \negg{27} & \textbf{27} & \negg{22} & \negg{26} & \negg{25} \\
scanalyzer & \bestbest{\textbf{50}} & \negg{48} & \negg{36} & \negg{48} & \bestbest{\textbf{50}} & \negg{48} & \negg{48} & \negg{46} & \textbf{46} & \textbf{46} & \negg{44} & \textbf{46} \\
schedule & 78 & \negg{63} & \textbf{\poss{106}} & \poss{98} & 19 & \poss{77} & \textbf{\poss{128}} & \poss{123} & 31 & \poss{77} & \bestbest{\textbf{{143}}} & \poss{140} \\
settlers & 0 & 0 & 0 & 0 & \bestbest{\textbf{5}} & \negg{3} & \negg{3} & \negg{4} & 0 & 0 & \poss{1} & \textbf{\poss{2}} \\
slitherlink & 0 & 0 & 0 & 0 & 0 & 0 & 0 & 0 & 0 & 0 & 0 & 0 \\
snake & 4 & \textbf{\poss{6}} & \negg{3} & \negg{2} & 3 & \poss{4} & \bestbest{\textbf{{11}}} & \poss{9} & 5 & \poss{6} & \poss{9} & \textbf{\poss{10}} \\
sokoban & \textbf{36} & \negg{28} & \negg{15} & \negg{15} & \bestbest{\textbf{46}} & \negg{45} & \negg{23} & \negg{38} & \bestbest{\textbf{46}} & \negg{42} & \negg{22} & \negg{34} \\
spider & 11 & \textbf{\poss{12}} & \negg{10} & \negg{9} & 7 & \negg{2} & \textbf{\poss{12}} & \poss{11} & 12 & \poss{14} & \poss{14} & \bestbest{\textbf{{18}}} \\
storage & 18 & \textbf{\poss{20}} & \negg{16} & \textbf{\poss{20}} & 16 & \poss{26} & \poss{18} & \bestbest{\textbf{{27}}} & 18 & \textbf{\poss{26}} & 18 & \poss{24} \\
termes & \textbf{9} & \negg{1} & \negg{0} & \negg{0} & \textbf{4} & \negg{0} & \negg{1} & \negg{1} & \bestbest{\textbf{12}} & \negg{1} & \negg{0} & \negg{1} \\
tetris & \bestbest{\textbf{19}} & \negg{16} & \negg{16} & \negg{13} & \textbf{16} & \negg{14} & \negg{10} & \negg{12} & 4 & \poss{13} & 4 & \textbf{\poss{15}} \\
thoughtful & \textbf{5} & \textbf{5} & \textbf{5} & \textbf{5} & 13 & \bestbest{\textbf{{19}}} & \poss{15} & \poss{16} & 8 & \poss{17} & \poss{13} & \bestbest{\textbf{{19}}} \\
tidybot & 19 & \bestbest{\textbf{{20}}} & 19 & \bestbest{\textbf{{20}}} & 17 & 17 & \textbf{\poss{18}} & 17 & 16 & \poss{17} & \textbf{\poss{18}} & \textbf{\poss{18}} \\
tpp & 13 & \textbf{\poss{15}} & \textbf{\poss{15}} & \textbf{\poss{15}} & 21 & \bestbest{\textbf{{28}}} & \negg{20} & 21 & 19 & \bestbest{\textbf{{28}}} & \poss{20} & 19 \\
transport & \bestbest{\textbf{41}} & \negg{27} & \negg{27} & \negg{26} & 36 & \negg{25} & \textbf{\poss{40}} & 36 & 12 & \poss{27} & \poss{21} & \textbf{\poss{34}} \\
trucks & 9 & \poss{12} & \textbf{\poss{16}} & \poss{13} & 15 & \negg{14} & \poss{20} & \textbf{\poss{21}} & 14 & \poss{18} & \poss{20} & \bestbest{\textbf{{23}}} \\
visitall & \bestbest{\textbf{40}} & \negg{21} & \negg{37} & \negg{12} & 3 & \poss{7} & \poss{5} & \textbf{\poss{10}} & 3 & \textbf{\poss{16}} & \negg{1} & \poss{14} \\
woodworking & 15 & \textbf{\poss{25}} & \poss{19} & \poss{17} & 35 & \poss{48} & \poss{48} & \textbf{\poss{49}} & 41 & \poss{46} & \bestbest{\textbf{{50}}} & \bestbest{\textbf{{50}}} \\
zenotravel & \bestbest{\textbf{20}} & \bestbest{\textbf{20}} & \bestbest{\textbf{20}} & \bestbest{\textbf{20}} & \bestbest{\textbf{20}} & \bestbest{\textbf{20}} & \bestbest{\textbf{20}} & \bestbest{\textbf{20}} & \bestbest{\textbf{20}} & \bestbest{\textbf{20}} & \bestbest{\textbf{20}} & \bestbest{\textbf{20}} \\
\midsplitterB
$\Sigma$ IPC & \textbf{1447} & \negg{1406} & \negg{1356} & \negg{1346} & 1551 & \poss{1693} & \poss{1694} & \textbf{\poss{1757}} & 1535 & \poss{1676} & \poss{1668} & \bestbest{\textbf{{1777}}} \\
\midsplitterB
blocksworld & 9 & \negg{2} & \bestbest{\textbf{{15}}} & \poss{13} & 0 & \poss{1} & \textbf{\poss{3}} & \textbf{\poss{3}} & 2 & \negg{1} & \textbf{\poss{3}} & 2 \\
childsnack & \textbf{19} & \negg{2} & \negg{17} & \negg{14} & 23 & 23 & \textbf{\poss{64}} & \poss{62} & 23 & 23 & \poss{62} & \bestbest{\textbf{{65}}} \\
genome & \bestbest{\textbf{312}} & \negg{252} & \negg{260} & \negg{261} & 100 & \textbf{\poss{106}} & \negg{89} & \negg{81} & 45 & \negg{43} & \textbf{\poss{63}} & \poss{57} \\
labyrinth & \bestbest{\textbf{40}} & \bestbest{\textbf{40}} & \bestbest{\textbf{40}} & \bestbest{\textbf{40}} & \bestbest{\textbf{40}} & \bestbest{\textbf{40}} & \bestbest{\textbf{40}} & \bestbest{\textbf{40}} & \bestbest{\textbf{40}} & \bestbest{\textbf{40}} & \bestbest{\textbf{40}} & \bestbest{\textbf{40}} \\
logistics & 3 & \negg{1} & \bestbest{\textbf{{31}}} & 3 & \textbf{1} & \negg{0} & \negg{0} & \negg{0} & 0 & 0 & 0 & \textbf{\poss{2}} \\
organic & 44 & \bestbest{\textbf{{47}}} & \poss{45} & \poss{45} & 31 & \poss{32} & \negg{30} & \textbf{\poss{34}} & 33 & \textbf{\poss{34}} & \negg{32} & \negg{32} \\
pipesworld & 23 & \poss{31} & 23 & \bestbest{\textbf{{34}}} & 19 & \textbf{\poss{28}} & \poss{27} & \poss{26} & 17 & \textbf{\poss{30}} & \poss{22} & \poss{29} \\
rovers & 0 & 0 & \textbf{\poss{2}} & \poss{1} & \textbf{7} & \negg{0} & \negg{3} & \negg{3} & \bestbest{\textbf{11}} & \negg{0} & \negg{3} & \negg{3} \\
visitall & 62 & \bestbest{\textbf{{111}}} & \poss{69} & \poss{67} & \textbf{96} & \negg{69} & \negg{85} & \negg{90} & \textbf{94} & \negg{74} & \negg{84} & \negg{87} \\
\midsplitterB
$\Sigma$ HTG & \bestbest{\textbf{512}} & \negg{486} & \negg{502} & \negg{478} & 317 & \negg{299} & \textbf{\poss{341}} & \poss{339} & 265 & \negg{245} & \poss{309} & \textbf{\poss{317}} \\
\midsplitterB
$\Sigma$ & \textbf{1959} & \negg{1892} & \negg{1858} & \negg{1824} & 1868 & \poss{1992} & \poss{2035} & \bestbest{\textbf{{2096}}} & 1800 & \poss{1921} & \poss{1977} & \textbf{\poss{2094}} \\
    \bottomrule
\end{tabularx}

  \caption{Coverage ($\uparrow$) of heuristics for single-queue GBFS across different domains. Green/red cells indicate that a novelty heuristic (\atoms, \wl, \atomswl) performs better/worse than their base heuristic (\hgc, \hadd, \hff). Blue cells indicate the best value per row. Bold font indicates the best score in each local heuristic group.}
  \label{tab:cov}
\end{table*}

\begin{table*}[ht!]
  \centering
  \tiny

\newcommand{\midsplitterA}{
\cmidrule(l){2-5}
\cmidrule(l){6-9}
\cmidrule(l){10-13}
}
\newcommand{\midsplitterB}{
\cmidrule{1-1}
\cmidrule(l){2-5}
\cmidrule(l){6-9}
\cmidrule(l){10-13}
}
\begin{tabularx}{\textwidth}{X *{12}{Y}}
    \toprule
    & \multicolumn{4}{c}{\hgc (least informative)}
    & \multicolumn{4}{c}{\hadd (informative)}
    & \multicolumn{4}{c}{\hff (most informative)}
    \\
    \midsplitterA
    Domain 
    & Base & \atoms & \wl & \atomswl
    & Base & \atoms & \wl & \atomswl
    & Base & \atoms & \wl & \atomswl
    \\
    \midsplitterB
    agricola & 0.00 & \textbf{\poss{0.23}} & 0.00 & 0.00 & \bestbest{\textbf{0.85}} & \negg{0.31} & \negg{0.00} & \negg{0.23} & \textbf{0.46} & \negg{0.38} & \negg{0.00} & \negg{0.23} \\
assembly & 0.00 & 0.00 & \textbf{\poss{0.33}} & \poss{0.23} & 0.37 & \bestbest{\textbf{{1.00}}} & \poss{0.90} & \poss{0.77} & \bestbest{\textbf{1.00}} & \bestbest{\textbf{1.00}} & \negg{0.83} & \negg{0.90} \\
barman & 0.25 & \negg{0.00} & \textbf{\poss{0.31}} & \negg{0.00} & 0.00 & 0.00 & \textbf{\poss{0.06}} & 0.00 & 0.12 & \bestbest{\textbf{{0.69}}} & \poss{0.25} & \poss{0.25} \\
blocks & \bestbest{\textbf{1.00}} & \negg{0.94} & \negg{0.97} & \bestbest{\textbf{1.00}} & \bestbest{\textbf{1.00}} & \bestbest{\textbf{1.00}} & \bestbest{\textbf{1.00}} & \bestbest{\textbf{1.00}} & \bestbest{\textbf{1.00}} & \bestbest{\textbf{1.00}} & \bestbest{\textbf{1.00}} & \bestbest{\textbf{1.00}} \\
caldera & \textbf{0.70} & \negg{0.57} & \negg{0.61} & \negg{0.43} & \bestbest{\textbf{0.78}} & \bestbest{\textbf{0.78}} & \negg{0.74} & \negg{0.70} & 0.52 & \poss{0.70} & 0.52 & \textbf{\poss{0.74}} \\
cavediving & \bestbest{\textbf{1.00}} & \bestbest{\textbf{1.00}} & \bestbest{\textbf{1.00}} & \bestbest{\textbf{1.00}} & \bestbest{\textbf{1.00}} & \negg{0.86} & \bestbest{\textbf{1.00}} & \bestbest{\textbf{1.00}} & \bestbest{\textbf{1.00}} & \bestbest{\textbf{1.00}} & \bestbest{\textbf{1.00}} & \bestbest{\textbf{1.00}} \\
childsnack & 0.00 & 0.00 & 0.00 & 0.00 & 0.50 & \negg{0.25} & \negg{0.25} & \bestbest{\textbf{{1.00}}} & 0.00 & 0.00 & \poss{0.25} & \textbf{\poss{0.75}} \\
citycar & 0.00 & 0.00 & \poss{0.25} & \textbf{\poss{0.40}} & \bestbest{\textbf{1.00}} & \negg{0.75} & \bestbest{\textbf{1.00}} & \bestbest{\textbf{1.00}} & 0.00 & \poss{0.40} & \textbf{\poss{0.95}} & \textbf{\poss{0.95}} \\
data & 0.17 & 0.17 & \bestbest{\textbf{{0.67}}} & \poss{0.50} & 0.17 & \textbf{\poss{0.50}} & \textbf{\poss{0.50}} & \textbf{\poss{0.50}} & \bestbest{\textbf{0.67}} & \negg{0.50} & \negg{0.33} & \bestbest{\textbf{0.67}} \\
depot & 0.67 & \poss{0.76} & \negg{0.62} & \textbf{\poss{0.81}} & 0.62 & \poss{0.90} & \bestbest{\textbf{{1.00}}} & \bestbest{\textbf{{1.00}}} & 0.71 & \poss{0.81} & \poss{0.90} & \bestbest{\textbf{{1.00}}} \\
driverlog & \textbf{0.95} & \textbf{0.95} & \negg{0.90} & \textbf{0.95} & 0.90 & \textbf{\poss{0.95}} & \negg{0.85} & \textbf{\poss{0.95}} & 0.85 & \poss{0.95} & \poss{0.90} & \bestbest{\textbf{{1.00}}} \\
elevators & 0.85 & \bestbest{\textbf{{1.00}}} & \negg{0.62} & \negg{0.74} & 0.38 & \textbf{\poss{0.77}} & \textbf{\poss{0.77}} & \textbf{\poss{0.77}} & 0.28 & \negg{0.21} & \poss{0.41} & \textbf{\poss{0.49}} \\
floortile & 0.00 & 0.00 & \textbf{\poss{0.27}} & \poss{0.20} & \bestbest{\textbf{0.67}} & \negg{0.60} & \negg{0.53} & \negg{0.53} & \bestbest{\textbf{0.67}} & \negg{0.40} & \bestbest{\textbf{0.67}} & \bestbest{\textbf{0.67}} \\
folding & 0.73 & \bestbest{\textbf{{0.82}}} & \negg{0.55} & \negg{0.55} & 0.00 & 0.00 & \textbf{\poss{0.27}} & \textbf{\poss{0.27}} & 0.00 & \poss{0.18} & \textbf{\poss{0.27}} & \textbf{\poss{0.27}} \\
freecell & 0.58 & \textbf{\poss{0.71}} & \negg{0.54} & \poss{0.66} & \bestbest{\textbf{1.00}} & \negg{0.96} & \negg{0.94} & \negg{0.99} & \textbf{0.99} & \textbf{0.99} & \negg{0.95} & \negg{0.96} \\
ged & \bestbest{\textbf{1.00}} & \bestbest{\textbf{1.00}} & \negg{0.95} & \negg{0.95} & 0.00 & \textbf{\poss{0.55}} & \poss{0.40} & \poss{0.45} & 0.00 & 0.00 & \textbf{\poss{0.15}} & \poss{0.05} \\
grid & \textbf{0.60} & \textbf{0.60} & \textbf{0.60} & \textbf{0.60} & 0.60 & \bestbest{\textbf{{1.00}}} & \negg{0.40} & \bestbest{\textbf{{1.00}}} & 0.80 & \bestbest{\textbf{{1.00}}} & \negg{0.40} & \bestbest{\textbf{{1.00}}} \\
gripper & \bestbest{\textbf{1.00}} & \bestbest{\textbf{1.00}} & \bestbest{\textbf{1.00}} & \bestbest{\textbf{1.00}} & \bestbest{\textbf{1.00}} & \bestbest{\textbf{1.00}} & \bestbest{\textbf{1.00}} & \bestbest{\textbf{1.00}} & \bestbest{\textbf{1.00}} & \bestbest{\textbf{1.00}} & \bestbest{\textbf{1.00}} & \bestbest{\textbf{1.00}} \\
hiking & \textbf{0.10} & \textbf{0.10} & \textbf{0.10} & \textbf{0.10} & 0.95 & \negg{0.90} & \bestbest{\textbf{{1.00}}} & 0.95 & \bestbest{\textbf{1.00}} & \bestbest{\textbf{1.00}} & \bestbest{\textbf{1.00}} & \bestbest{\textbf{1.00}} \\
labyrinth & \bestbest{\textbf{1.00}} & \negg{0.50} & \negg{0.75} & \negg{0.50} & 0.00 & 0.00 & 0.00 & 0.00 & 0.00 & 0.00 & 0.00 & 0.00 \\
logistics & 0.61 & \textbf{\poss{0.63}} & \negg{0.60} & \textbf{\poss{0.63}} & \textbf{0.93} & \negg{0.84} & \negg{0.89} & \negg{0.91} & \bestbest{\textbf{0.96}} & \negg{0.88} & \negg{0.93} & \negg{0.91} \\
maintenance & 0.82 & \textbf{\poss{0.94}} & \negg{0.41} & \negg{0.53} & 0.94 & \bestbest{\textbf{{1.00}}} & \negg{0.59} & \negg{0.76} & \textbf{0.65} & \negg{0.41} & \negg{0.35} & \negg{0.35} \\
miconic & \textbf{0.85} & \textbf{0.85} & \negg{0.83} & \negg{0.83} & \bestbest{\textbf{1.00}} & \negg{0.99} & \negg{0.97} & \bestbest{\textbf{1.00}} & \bestbest{\textbf{1.00}} & \bestbest{\textbf{1.00}} & \bestbest{\textbf{1.00}} & \negg{0.99} \\
movie & \bestbest{\textbf{1.00}} & \bestbest{\textbf{1.00}} & \bestbest{\textbf{1.00}} & \bestbest{\textbf{1.00}} & \bestbest{\textbf{1.00}} & \bestbest{\textbf{1.00}} & \bestbest{\textbf{1.00}} & \bestbest{\textbf{1.00}} & \bestbest{\textbf{1.00}} & \bestbest{\textbf{1.00}} & \bestbest{\textbf{1.00}} & \bestbest{\textbf{1.00}} \\
mprime & 0.60 & \textbf{\poss{0.63}} & 0.60 & 0.60 & 0.89 & \bestbest{\textbf{{1.00}}} & \bestbest{\textbf{{1.00}}} & \bestbest{\textbf{{1.00}}} & 0.86 & \bestbest{\textbf{{1.00}}} & \bestbest{\textbf{{1.00}}} & \bestbest{\textbf{{1.00}}} \\
mystery & 0.79 & 0.79 & \negg{0.74} & \textbf{\poss{0.84}} & 0.95 & \bestbest{\textbf{{1.00}}} & \bestbest{\textbf{{1.00}}} & \bestbest{\textbf{{1.00}}} & 0.89 & \bestbest{\textbf{{1.00}}} & \bestbest{\textbf{{1.00}}} & \bestbest{\textbf{{1.00}}} \\
nomystery & 0.33 & \textbf{\poss{0.56}} & \poss{0.50} & \poss{0.44} & 0.33 & \textbf{\poss{0.67}} & 0.33 & \poss{0.61} & 0.50 & \bestbest{\textbf{{0.94}}} & \negg{0.44} & \poss{0.83} \\
nurikabe & \bestbest{\textbf{0.93}} & \bestbest{\textbf{0.93}} & \negg{0.86} & \negg{0.86} & 0.64 & \negg{0.50} & \negg{0.57} & \textbf{\poss{0.71}} & 0.50 & \textbf{\poss{0.79}} & 0.50 & \poss{0.64} \\
openstacks & \bestbest{\textbf{0.70}} & \negg{0.60} & \negg{0.58} & \negg{0.49} & 0.51 & \negg{0.49} & \textbf{\poss{0.60}} & \textbf{\poss{0.60}} & 0.63 & 0.63 & \textbf{\poss{0.67}} & \textbf{\poss{0.67}} \\
optical & 0.24 & \negg{0.19} & \textbf{\poss{0.86}} & \textbf{\poss{0.86}} & 0.29 & \bestbest{\textbf{{1.00}}} & \poss{0.76} & \poss{0.71} & 0.19 & \textbf{\poss{0.62}} & \poss{0.48} & \poss{0.33} \\
parking & 0.00 & 0.00 & 0.00 & 0.00 & 0.28 & \poss{0.44} & \poss{0.58} & \textbf{\poss{0.64}} & 0.64 & \negg{0.14} & \poss{0.81} & \bestbest{\textbf{{0.86}}} \\
pegsol & \bestbest{\textbf{1.00}} & \negg{0.92} & \negg{0.88} & \negg{0.88} & \bestbest{\textbf{1.00}} & \negg{0.88} & \negg{0.92} & \negg{0.96} & \bestbest{\textbf{1.00}} & \negg{0.92} & \negg{0.96} & \negg{0.96} \\
philosophers & 0.42 & \textbf{\poss{0.54}} & \negg{0.23} & \negg{0.25} & \textbf{0.98} & \negg{0.58} & \negg{0.50} & \negg{0.46} & \bestbest{\textbf{1.00}} & \bestbest{\textbf{1.00}} & \negg{0.35} & \negg{0.50} \\
pipesworld & 0.71 & \textbf{\poss{0.76}} & \negg{0.60} & \negg{0.70} & 0.50 & \poss{0.93} & \poss{0.88} & \bestbest{\textbf{{0.98}}} & 0.58 & \poss{0.94} & \poss{0.86} & \bestbest{\textbf{{0.98}}} \\
psr & \bestbest{\textbf{1.00}} & \negg{0.86} & \negg{0.75} & \negg{0.75} & \textbf{0.89} & \negg{0.85} & \negg{0.75} & \negg{0.75} & 0.63 & \negg{0.61} & \textbf{\poss{0.79}} & \textbf{\poss{0.79}} \\
recharging & \bestbest{\textbf{0.86}} & \negg{0.79} & \negg{0.57} & \negg{0.50} & 0.64 & \textbf{\poss{0.71}} & \negg{0.57} & \negg{0.50} & 0.57 & \textbf{\poss{0.79}} & 0.57 & \negg{0.50} \\
ricochet & \textbf{0.29} & \negg{0.00} & \negg{0.14} & \negg{0.00} & \textbf{0.29} & \negg{0.21} & \negg{0.00} & \negg{0.07} & \bestbest{\textbf{0.86}} & \negg{0.36} & \negg{0.14} & \negg{0.07} \\
rovers & \textbf{0.66} & \textbf{0.66} & \negg{0.53} & \negg{0.59} & 0.81 & \textbf{\poss{0.94}} & 0.81 & \poss{0.88} & 0.75 & \bestbest{\textbf{{0.97}}} & \poss{0.94} & \poss{0.91} \\
rubiks & 0.20 & \textbf{\poss{0.30}} & 0.20 & 0.20 & 0.20 & \textbf{\poss{0.40}} & \poss{0.25} & \poss{0.25} & \bestbest{\textbf{1.00}} & \negg{0.45} & \negg{0.30} & \negg{0.30} \\
satellite & 0.47 & \textbf{\poss{0.60}} & \negg{0.43} & 0.47 & \bestbest{\textbf{1.00}} & \negg{0.87} & \negg{0.93} & \negg{0.90} & \textbf{0.90} & \negg{0.73} & \negg{0.87} & \negg{0.83} \\
scanalyzer & \bestbest{\textbf{1.00}} & \negg{0.96} & \negg{0.72} & \negg{0.96} & \bestbest{\textbf{1.00}} & \negg{0.96} & \negg{0.96} & \negg{0.92} & \textbf{0.92} & \textbf{0.92} & \negg{0.88} & \textbf{0.92} \\
schedule & 0.52 & \negg{0.42} & \textbf{\poss{0.71}} & \poss{0.66} & 0.13 & \poss{0.52} & \textbf{\poss{0.86}} & \poss{0.83} & 0.21 & \poss{0.52} & \bestbest{\textbf{{0.96}}} & \poss{0.94} \\
settlers & 0.00 & 0.00 & 0.00 & 0.00 & \bestbest{\textbf{0.71}} & \negg{0.43} & \negg{0.43} & \negg{0.57} & 0.00 & 0.00 & \poss{0.14} & \textbf{\poss{0.29}} \\
slitherlink & 0.00 & 0.00 & 0.00 & 0.00 & 0.00 & 0.00 & 0.00 & 0.00 & 0.00 & 0.00 & 0.00 & 0.00 \\
snake & 0.36 & \textbf{\poss{0.55}} & \negg{0.27} & \negg{0.18} & 0.27 & \poss{0.36} & \bestbest{\textbf{{1.00}}} & \poss{0.82} & 0.45 & \poss{0.55} & \poss{0.82} & \textbf{\poss{0.91}} \\
sokoban & \textbf{0.75} & \negg{0.58} & \negg{0.31} & \negg{0.31} & \bestbest{\textbf{0.96}} & \negg{0.94} & \negg{0.48} & \negg{0.79} & \bestbest{\textbf{0.96}} & \negg{0.88} & \negg{0.46} & \negg{0.71} \\
spider & 0.58 & \textbf{\poss{0.63}} & \negg{0.53} & \negg{0.47} & 0.37 & \negg{0.11} & \textbf{\poss{0.63}} & \poss{0.58} & 0.63 & \poss{0.74} & \poss{0.74} & \bestbest{\textbf{{0.95}}} \\
storage & 0.64 & \textbf{\poss{0.71}} & \negg{0.57} & \textbf{\poss{0.71}} & 0.57 & \poss{0.93} & \poss{0.64} & \bestbest{\textbf{{0.96}}} & 0.64 & \textbf{\poss{0.93}} & 0.64 & \poss{0.86} \\
termes & \textbf{0.75} & \negg{0.08} & \negg{0.00} & \negg{0.00} & \textbf{0.33} & \negg{0.00} & \negg{0.08} & \negg{0.08} & \bestbest{\textbf{1.00}} & \negg{0.08} & \negg{0.00} & \negg{0.08} \\
tetris & \bestbest{\textbf{0.95}} & \negg{0.80} & \negg{0.80} & \negg{0.65} & \textbf{0.80} & \negg{0.70} & \negg{0.50} & \negg{0.60} & 0.20 & \poss{0.65} & 0.20 & \textbf{\poss{0.75}} \\
thoughtful & \textbf{0.25} & \textbf{0.25} & \textbf{0.25} & \textbf{0.25} & 0.65 & \bestbest{\textbf{{0.95}}} & \poss{0.75} & \poss{0.80} & 0.40 & \poss{0.85} & \poss{0.65} & \bestbest{\textbf{{0.95}}} \\
tidybot & 0.95 & \bestbest{\textbf{{1.00}}} & 0.95 & \bestbest{\textbf{{1.00}}} & 0.85 & 0.85 & \textbf{\poss{0.90}} & 0.85 & 0.80 & \poss{0.85} & \textbf{\poss{0.90}} & \textbf{\poss{0.90}} \\
tpp & 0.46 & \textbf{\poss{0.54}} & \textbf{\poss{0.54}} & \textbf{\poss{0.54}} & 0.75 & \bestbest{\textbf{{1.00}}} & \negg{0.71} & 0.75 & 0.68 & \bestbest{\textbf{{1.00}}} & \poss{0.71} & 0.68 \\
transport & \bestbest{\textbf{0.87}} & \negg{0.57} & \negg{0.57} & \negg{0.55} & 0.77 & \negg{0.53} & \textbf{\poss{0.85}} & 0.77 & 0.26 & \poss{0.57} & \poss{0.45} & \textbf{\poss{0.72}} \\
trucks & 0.36 & \poss{0.48} & \textbf{\poss{0.64}} & \poss{0.52} & 0.60 & \negg{0.56} & \poss{0.80} & \textbf{\poss{0.84}} & 0.56 & \poss{0.72} & \poss{0.80} & \bestbest{\textbf{{0.92}}} \\
visitall & \bestbest{\textbf{1.00}} & \negg{0.53} & \negg{0.93} & \negg{0.30} & 0.08 & \poss{0.18} & \poss{0.12} & \textbf{\poss{0.25}} & 0.08 & \textbf{\poss{0.40}} & \negg{0.03} & \poss{0.35} \\
woodworking & 0.30 & \textbf{\poss{0.50}} & \poss{0.38} & \poss{0.34} & 0.70 & \poss{0.96} & \poss{0.96} & \textbf{\poss{0.98}} & 0.82 & \poss{0.92} & \bestbest{\textbf{{1.00}}} & \bestbest{\textbf{{1.00}}} \\
zenotravel & \bestbest{\textbf{1.00}} & \bestbest{\textbf{1.00}} & \bestbest{\textbf{1.00}} & \bestbest{\textbf{1.00}} & \bestbest{\textbf{1.00}} & \bestbest{\textbf{1.00}} & \bestbest{\textbf{1.00}} & \bestbest{\textbf{1.00}} & \bestbest{\textbf{1.00}} & \bestbest{\textbf{1.00}} & \bestbest{\textbf{1.00}} & \bestbest{\textbf{1.00}} \\
\midsplitterB
$\Sigma$ IPC & \textbf{33.80} & \negg{32.48} & \negg{31.51} & \negg{30.50} & 36.38 & \poss{39.36} & \poss{38.22} & \textbf{\poss{41.19}} & 35.79 & \poss{38.95} & \poss{36.13} & \bestbest{\textbf{{41.29}}} \\
\midsplitterB
blocksworld & 0.60 & \negg{0.13} & \bestbest{\textbf{{1.00}}} & \poss{0.87} & 0.00 & \poss{0.07} & \textbf{\poss{0.20}} & \textbf{\poss{0.20}} & 0.13 & \negg{0.07} & \textbf{\poss{0.20}} & 0.13 \\
childsnack & \textbf{0.29} & \negg{0.03} & \negg{0.26} & \negg{0.22} & 0.35 & 0.35 & \textbf{\poss{0.98}} & \poss{0.95} & 0.35 & 0.35 & \poss{0.95} & \bestbest{\textbf{{1.00}}} \\
genome & \bestbest{\textbf{1.00}} & \negg{0.81} & \negg{0.83} & \negg{0.84} & 0.32 & \textbf{\poss{0.34}} & \negg{0.29} & \negg{0.26} & 0.14 & 0.14 & \textbf{\poss{0.20}} & \poss{0.18} \\
labyrinth & \bestbest{\textbf{1.00}} & \bestbest{\textbf{1.00}} & \bestbest{\textbf{1.00}} & \bestbest{\textbf{1.00}} & \bestbest{\textbf{1.00}} & \bestbest{\textbf{1.00}} & \bestbest{\textbf{1.00}} & \bestbest{\textbf{1.00}} & \bestbest{\textbf{1.00}} & \bestbest{\textbf{1.00}} & \bestbest{\textbf{1.00}} & \bestbest{\textbf{1.00}} \\
logistics & 0.09 & \negg{0.03} & \bestbest{\textbf{{0.97}}} & 0.09 & \textbf{0.03} & \negg{0.00} & \negg{0.00} & \negg{0.00} & 0.00 & 0.00 & 0.00 & \textbf{\poss{0.06}} \\
organic & 0.92 & \bestbest{\textbf{{0.98}}} & \poss{0.94} & \poss{0.94} & 0.65 & \poss{0.67} & \negg{0.62} & \textbf{\poss{0.71}} & 0.69 & \textbf{\poss{0.71}} & \negg{0.67} & \negg{0.67} \\
pipesworld & 0.55 & \poss{0.74} & 0.55 & \bestbest{\textbf{{0.81}}} & 0.45 & \textbf{\poss{0.67}} & \poss{0.64} & \poss{0.62} & 0.40 & \textbf{\poss{0.71}} & \poss{0.52} & \poss{0.69} \\
rovers & 0.00 & 0.00 & \textbf{\poss{0.18}} & \poss{0.09} & \textbf{0.64} & \negg{0.00} & \negg{0.27} & \negg{0.27} & \bestbest{\textbf{1.00}} & \negg{0.00} & \negg{0.27} & \negg{0.27} \\
visitall & 0.48 & \bestbest{\textbf{{0.87}}} & \poss{0.54} & \poss{0.52} & \textbf{0.75} & \negg{0.54} & \negg{0.66} & \negg{0.70} & \textbf{0.73} & \negg{0.58} & \negg{0.66} & \negg{0.68} \\
\midsplitterB
$\Sigma$ HTG & 4.93 & \negg{4.59} & \bestbest{\textbf{{6.27}}} & \poss{5.37} & 4.19 & \negg{3.63} & \poss{4.67} & \textbf{\poss{4.72}} & 4.46 & \negg{3.56} & \poss{4.48} & \textbf{\poss{4.69}} \\
\midsplitterB
$\Sigma$ & \textbf{38.73} & \negg{37.07} & \negg{37.78} & \negg{35.87} & 40.57 & \poss{43.00} & \poss{42.90} & \textbf{\poss{45.91}} & 40.25 & \poss{42.51} & \poss{40.61} & \bestbest{\textbf{{45.97}}} \\
    \bottomrule
\end{tabularx}

  \caption{Normalised coverage ($\uparrow$) of heuristics for single-queue GBFS across different domains. Green/red cells indicate that a novelty heuristic (\atoms, \wl, \atomswl) performs better/worse than their base heuristic (\hgc, \hadd, \hff). Blue cells indicate the best value per row. Bold font indicates the best score in each local heuristic group.}
  \label{tab:cov_norm}
\end{table*}

\clearpage
\section{Expansion Results}\label{sec:expansions}

\begin{table}[ht!]
  \centering
  \tabcolsep 3pt
  \renewcommand{\poss}[1]{#1}
  \renewcommand{\negg}[1]{#1}
  \newcommand{\midsplitterB}{
    \cmidrule{1-1}
    \cmidrule(l){2-3}
    \cmidrule(l){4-5}
    \cmidrule(l){6-7}
    \cmidrule(l){8-9}
  }
  \tiny
  \begin{tabularx}{\columnwidth}{l | *{2}{Y} | *{2}{Y} | *{2}{Y} | *{2}{Y}}
    \toprule
    Domain 
    & Base
    & \atoms
    & Base
    & \wl
    & Base
    & \atomswl
    & \atoms
    & \wl
    \\
    \midsplitterB
    agricola & 2 & \textbf{\possA{5}} & \textbf{6} & \negg{0} & \textbf{3} & \textbf{3} & \textbf{5} & \negg{0} \\
assembly & \textbf{30} & \negg{0} & \textbf{29} & \negg{1} & \textbf{29} & \negg{1} & \textbf{25} & \negg{5} \\
barman & 1 & \textbf{\possA{11}} & 2 & \textbf{\possA{4}} & 2 & \textbf{\possA{4}} & \textbf{9} & \negg{4} \\
blocks & 15 & \textbf{\possA{20}} & 12 & \textbf{\possA{23}} & 11 & \textbf{\possA{23}} & 9 & \textbf{\possA{24}} \\
caldera & 8 & \textbf{\possA{9}} & \textbf{7} & \textbf{7} & 6 & \textbf{\possA{10}} & 6 & \textbf{\possA{10}} \\
cavediving & 0 & \textbf{\possA{7}} & 0 & \textbf{\possA{7}} & 0 & \textbf{\possA{7}} & 0 & \textbf{\possA{7}} \\
childsnack & 0 & 0 & 0 & \textbf{\possA{1}} & 0 & \textbf{\possA{3}} & 0 & \textbf{\possA{1}} \\
citycar & 0 & \textbf{\possA{8}} & 0 & \textbf{\possA{19}} & 0 & \textbf{\possA{19}} & 1 & \textbf{\possA{18}} \\
data & 1 & \textbf{\possA{3}} & \textbf{2} & \textbf{2} & 1 & \textbf{\possA{4}} & \textbf{3} & \negg{0} \\
depot & 8 & \textbf{\possA{10}} & 5 & \textbf{\possA{15}} & 3 & \textbf{\possA{18}} & 3 & \textbf{\possA{17}} \\
driverlog & \textbf{16} & \negg{3} & \textbf{17} & \negg{3} & \textbf{16} & \negg{4} & 5 & \textbf{\possA{15}} \\
elevators & 5 & \textbf{\possA{6}} & 0 & \textbf{\possA{16}} & 1 & \textbf{\possA{18}} & 2 & \textbf{\possA{14}} \\
floortile & \textbf{10} & \negg{0} & 4 & \textbf{\possA{7}} & \textbf{5} & \textbf{5} & 0 & \textbf{\possA{10}} \\
folding & 0 & \textbf{\possA{2}} & 0 & \textbf{\possA{3}} & 0 & \textbf{\possA{3}} & \textbf{2} & \textbf{2} \\
freecell & \textbf{61} & \negg{17} & \textbf{62} & \negg{17} & \textbf{57} & \negg{22} & \textbf{48} & \negg{32} \\
ged & 0 & 0 & 0 & \textbf{\possA{3}} & 0 & \textbf{\possA{1}} & 0 & \textbf{\possA{3}} \\
grid & 2 & \textbf{\possA{3}} & \textbf{4} & \negg{0} & 2 & \textbf{\possA{3}} & \textbf{4} & \negg{1} \\
gripper & \textbf{20} & \negg{0} & 0 & \textbf{\possA{20}} & 0 & \textbf{\possA{20}} & 0 & \textbf{\possA{20}} \\
hiking & \textbf{10} & \textbf{10} & 2 & \textbf{\possA{18}} & 2 & \textbf{\possA{18}} & 4 & \textbf{\possA{16}} \\
labyrinth & 0 & 0 & 0 & 0 & 0 & 0 & 0 & 0 \\
logistics & \textbf{53} & \negg{1} & \textbf{54} & \negg{2} & \textbf{53} & \negg{2} & 1 & \textbf{\possA{51}} \\
maintenance & \textbf{11} & \negg{3} & \textbf{11} & \negg{2} & \textbf{11} & \negg{2} & 1 & \textbf{\possA{6}} \\
miconic & \textbf{404} & \negg{17} & \textbf{397} & \negg{20} & \textbf{402} & \negg{19} & 138 & \textbf{\possA{280}} \\
movie & \textbf{30} & \negg{0} & \textbf{30} & \negg{0} & \textbf{30} & \negg{0} & 0 & 0 \\
mprime & 6 & \textbf{\possA{15}} & 9 & \textbf{\possA{13}} & 9 & \textbf{\possA{13}} & \textbf{11} & \negg{5} \\
mystery & \textbf{4} & \textbf{4} & \textbf{5} & \negg{3} & \textbf{4} & \negg{3} & \textbf{4} & \negg{1} \\
nomystery & 4 & \textbf{\possA{13}} & \textbf{6} & \negg{4} & 5 & \textbf{\possA{10}} & \textbf{14} & \negg{3} \\
nurikabe & 2 & \textbf{\possA{9}} & \textbf{5} & \negg{3} & 3 & \textbf{\possA{6}} & \textbf{8} & \negg{3} \\
openstacks & 2 & \textbf{\possA{34}} & 0 & \textbf{\possA{38}} & 0 & \textbf{\possA{38}} & 8 & \textbf{\possA{29}} \\
optical & 0 & \textbf{\possA{13}} & 0 & \textbf{\possA{10}} & 0 & \textbf{\possA{7}} & \textbf{10} & \negg{2} \\
parking & \textbf{22} & \negg{1} & 9 & \textbf{\possA{26}} & 5 & \textbf{\possA{26}} & 1 & \textbf{\possA{29}} \\
pegsol & \textbf{45} & \negg{5} & \textbf{33} & \negg{17} & \textbf{36} & \negg{14} & 15 & \textbf{\possA{32}} \\
philosophers & 18 & \textbf{\possA{30}} & \textbf{48} & \negg{0} & \textbf{48} & \negg{0} & \textbf{47} & \negg{1} \\
pipesworld & 18 & \textbf{\possA{60}} & 14 & \textbf{\possA{56}} & 15 & \textbf{\possA{65}} & 29 & \textbf{\possA{50}} \\
psr & \textbf{38} & \negg{21} & 24 & \textbf{\possA{47}} & 27 & \textbf{\possA{46}} & 22 & \textbf{\possA{51}} \\
recharging & 5 & \textbf{\possA{6}} & 4 & \textbf{\possA{6}} & 4 & \textbf{\possA{5}} & \textbf{7} & \negg{5} \\
ricochet & \textbf{10} & \negg{2} & \textbf{11} & \negg{2} & \textbf{11} & \negg{1} & \textbf{5} & \negg{1} \\
rovers & \textbf{16} & \negg{14} & \textbf{16} & \negg{12} & 13 & \textbf{\possA{15}} & 8 & \textbf{\possA{23}} \\
rubiks & \textbf{14} & \negg{2} & \textbf{19} & \negg{1} & \textbf{15} & \negg{1} & \textbf{9} & \negg{0} \\
satellite & \textbf{25} & \negg{3} & \textbf{19} & \negg{8} & \textbf{22} & \negg{5} & 1 & \textbf{\possA{25}} \\
scanalyzer & \textbf{37} & \negg{10} & \textbf{27} & \negg{23} & 23 & \textbf{\possA{24}} & 19 & \textbf{\possA{27}} \\
schedule & 8 & \textbf{\possA{65}} & 3 & \textbf{\possA{134}} & 3 & \textbf{\possA{131}} & 24 & \textbf{\possA{113}} \\
settlers & 0 & 0 & 0 & \textbf{\possA{1}} & 0 & \textbf{\possA{2}} & 0 & \textbf{\possA{1}} \\
snake & 2 & \textbf{\possA{5}} & 0 & \textbf{\possA{9}} & 0 & \textbf{\possA{10}} & 2 & \textbf{\possA{7}} \\
sokoban & \textbf{41} & \negg{7} & \textbf{43} & \negg{1} & \textbf{43} & \negg{5} & \textbf{37} & \negg{5} \\
spider & \textbf{9} & \negg{7} & 7 & \textbf{\possA{10}} & 6 & \textbf{\possA{13}} & 5 & \textbf{\possA{12}} \\
storage & 8 & \textbf{\possA{14}} & \textbf{12} & \negg{3} & 7 & \textbf{\possA{13}} & \textbf{18} & \negg{4} \\
termes & \textbf{11} & \negg{1} & \textbf{12} & \negg{0} & \textbf{11} & \negg{1} & \textbf{1} & \negg{0} \\
tetris & 1 & \textbf{\possA{13}} & \textbf{4} & \negg{2} & 1 & \textbf{\possA{14}} & \textbf{13} & \negg{0} \\
thoughtful & 8 & \textbf{\possA{9}} & \textbf{8} & \negg{6} & 8 & \textbf{\possA{11}} & \textbf{14} & \negg{3} \\
tidybot & \textbf{12} & \negg{5} & \textbf{13} & \negg{5} & \textbf{12} & \negg{6} & 6 & \textbf{\possA{12}} \\
tpp & 3 & \textbf{\possA{20}} & 5 & \textbf{\possA{11}} & 5 & \textbf{\possA{11}} & \textbf{21} & \negg{3} \\
transport & 5 & \textbf{\possA{21}} & 2 & \textbf{\possA{18}} & 1 & \textbf{\possA{32}} & 12 & \textbf{\possA{14}} \\
trucks & 9 & \textbf{\possA{10}} & 6 & \textbf{\possA{14}} & 5 & \textbf{\possA{18}} & 4 & \textbf{\possA{18}} \\
visitall & 0 & \textbf{\possA{16}} & \textbf{3} & \negg{0} & 0 & \textbf{\possA{14}} & \textbf{16} & \negg{0} \\
woodworking & \textbf{35} & \negg{12} & \textbf{33} & \negg{15} & \textbf{33} & \negg{15} & 6 & \textbf{\possA{39}} \\
zenotravel & \textbf{17} & \negg{0} & \textbf{14} & \negg{5} & \textbf{11} & \negg{6} & 3 & \textbf{\possA{16}} \\
\midsplitterB
$\Sigma$ IPC & \textbf{1122} & \negg{582} & \textbf{1058} & \negg{693} & \textbf{1020} & \negg{790} & 666 & \textbf{\possA{1070}} \\
\midsplitterB
blocksworld & \textbf{2} & \negg{0} & \textbf{1} & \textbf{1} & \textbf{1} & \negg{0} & 0 & \textbf{\possA{3}} \\
childsnack & 9 & \textbf{\possA{18}} & 1 & \textbf{\possA{61}} & 1 & \textbf{\possA{64}} & 0 & \textbf{\possA{62}} \\
genome & \textbf{28} & \negg{16} & 24 & \textbf{\possA{50}} & 23 & \textbf{\possA{45}} & 21 & \textbf{\possA{52}} \\
labyrinth & 2 & \textbf{\possA{7}} & \textbf{31} & \negg{9} & 0 & \textbf{\possA{9}} & \textbf{31} & \negg{9} \\
logistics & 0 & 0 & 0 & 0 & 0 & \textbf{\possA{2}} & 0 & 0 \\
organic & \textbf{3} & \negg{1} & \textbf{4} & \negg{1} & \textbf{4} & \negg{1} & \textbf{2} & \negg{1} \\
pipesworld & 10 & \textbf{\possA{19}} & 10 & \textbf{\possA{12}} & 9 & \textbf{\possA{20}} & 13 & \textbf{\possA{17}} \\
rovers & \textbf{11} & \negg{0} & \textbf{9} & \negg{0} & \textbf{10} & \negg{0} & 0 & \textbf{\possA{3}} \\
visitall & \textbf{46} & \negg{10} & \textbf{31} & \negg{28} & 24 & \textbf{\possA{33}} & 11 & \textbf{\possA{41}} \\
\midsplitterB
$\Sigma$ HTG & \textbf{111} & \negg{71} & 111 & \textbf{\possA{162}} & 72 & \textbf{\possA{174}} & 78 & \textbf{\possA{188}} \\
\midsplitterB
$\Sigma$ & \textbf{1233} & \negg{653} & \textbf{1169} & \negg{855} & \textbf{1092} & \negg{964} & 744 & \textbf{\possA{1258}} \\
    \bottomrule
\end{tabularx}

  \caption{Problems per domain where a column expands fewer nodes than its paired column. The higher value in each pair is indicated in bold. The \mbox{(non-)novelty} extensions are over the \hff{} heuristic.}
  \label{tab:versus-exp}
\end{table}

\begin{table}[ht!]
  \centering
  \tabcolsep 3pt
  \renewcommand{\poss}[1]{#1}
  \renewcommand{\negg}[1]{#1}
  \newcommand{\midsplitterB}{
    \cmidrule{1-1}
    \cmidrule(l){2-3}
    \cmidrule(l){4-5}
    \cmidrule(l){6-7}
    \cmidrule(l){8-9}
  }
  \tiny
  \begin{tabularx}{\columnwidth}{l | *{2}{Y} | *{2}{Y} | *{2}{Y} | *{2}{Y}}
    \toprule
    Domain 
    & Base
    & \atoms
    & Base
    & \wl
    & Base
    & \atomswl
    & \atoms
    & \wl
    \\
    \midsplitterB
    agricola & 0.10 & \textbf{\possA{0.25}} & \textbf{0.30} & \negg{0.00} & \textbf{0.15} & \textbf{0.15} & \textbf{0.25} & \negg{0.00} \\
assembly & \textbf{1.00} & \negg{0.00} & \textbf{0.97} & \negg{0.03} & \textbf{0.97} & \negg{0.03} & \textbf{0.83} & \negg{0.17} \\
barman & 0.03 & \textbf{\possA{0.28}} & 0.05 & \textbf{\possA{0.10}} & 0.05 & \textbf{\possA{0.10}} & \textbf{0.23} & \negg{0.10} \\
blocks & 0.43 & \textbf{\possA{0.57}} & 0.34 & \textbf{\possA{0.66}} & 0.31 & \textbf{\possA{0.66}} & 0.26 & \textbf{\possA{0.69}} \\
caldera & 0.20 & \textbf{\possA{0.23}} & \textbf{0.17} & \textbf{0.17} & 0.15 & \textbf{\possA{0.25}} & 0.15 & \textbf{\possA{0.25}} \\
cavediving & 0.00 & \textbf{\possA{0.35}} & 0.00 & \textbf{\possA{0.35}} & 0.00 & \textbf{\possA{0.35}} & 0.00 & \textbf{\possA{0.35}} \\
childsnack & 0.00 & 0.00 & 0.00 & \textbf{\possA{0.05}} & 0.00 & \textbf{\possA{0.15}} & 0.00 & \textbf{\possA{0.05}} \\
citycar & 0.00 & \textbf{\possA{0.40}} & 0.00 & \textbf{\possA{0.95}} & 0.00 & \textbf{\possA{0.95}} & 0.05 & \textbf{\possA{0.90}} \\
data & 0.05 & \textbf{\possA{0.15}} & \textbf{0.10} & \textbf{0.10} & 0.05 & \textbf{\possA{0.20}} & \textbf{0.15} & \negg{0.00} \\
depot & 0.36 & \textbf{\possA{0.45}} & 0.23 & \textbf{\possA{0.68}} & 0.14 & \textbf{\possA{0.82}} & 0.14 & \textbf{\possA{0.77}} \\
driverlog & \textbf{0.80} & \negg{0.15} & \textbf{0.85} & \negg{0.15} & \textbf{0.80} & \negg{0.20} & 0.25 & \textbf{\possA{0.75}} \\
elevators & 0.10 & \textbf{\possA{0.12}} & 0.00 & \textbf{\possA{0.32}} & 0.02 & \textbf{\possA{0.36}} & 0.04 & \textbf{\possA{0.28}} \\
floortile & \textbf{0.25} & \negg{0.00} & 0.10 & \textbf{\possA{0.17}} & \textbf{0.12} & \textbf{0.12} & 0.00 & \textbf{\possA{0.25}} \\
folding & 0.00 & \textbf{\possA{0.10}} & 0.00 & \textbf{\possA{0.15}} & 0.00 & \textbf{\possA{0.15}} & \textbf{0.10} & \textbf{0.10} \\
freecell & \textbf{0.76} & \negg{0.21} & \textbf{0.78} & \negg{0.21} & \textbf{0.71} & \negg{0.28} & \textbf{0.60} & \negg{0.40} \\
ged & 0.00 & 0.00 & 0.00 & \textbf{\possA{0.15}} & 0.00 & \textbf{\possA{0.05}} & 0.00 & \textbf{\possA{0.15}} \\
grid & 0.40 & \textbf{\possA{0.60}} & \textbf{0.80} & \negg{0.00} & 0.40 & \textbf{\possA{0.60}} & \textbf{0.80} & \negg{0.20} \\
gripper & \textbf{1.00} & \negg{0.00} & 0.00 & \textbf{\possA{1.00}} & 0.00 & \textbf{\possA{1.00}} & 0.00 & \textbf{\possA{1.00}} \\
hiking & \textbf{0.50} & \textbf{0.50} & 0.10 & \textbf{\possA{0.90}} & 0.10 & \textbf{\possA{0.90}} & 0.20 & \textbf{\possA{0.80}} \\
labyrinth & 0.00 & 0.00 & 0.00 & 0.00 & 0.00 & 0.00 & 0.00 & 0.00 \\
logistics & \textbf{0.84} & \negg{0.02} & \textbf{0.86} & \negg{0.03} & \textbf{0.84} & \negg{0.03} & 0.02 & \textbf{\possA{0.81}} \\
maintenance & \textbf{0.55} & \negg{0.15} & \textbf{0.55} & \negg{0.10} & \textbf{0.55} & \negg{0.10} & 0.05 & \textbf{\possA{0.30}} \\
miconic & \textbf{0.90} & \negg{0.04} & \textbf{0.88} & \negg{0.04} & \textbf{0.89} & \negg{0.04} & 0.31 & \textbf{\possA{0.62}} \\
movie & \textbf{1.00} & \negg{0.00} & \textbf{1.00} & \negg{0.00} & \textbf{1.00} & \negg{0.00} & 0.00 & 0.00 \\
mprime & 0.17 & \textbf{\possA{0.43}} & 0.26 & \textbf{\possA{0.37}} & 0.26 & \textbf{\possA{0.37}} & \textbf{0.31} & \negg{0.14} \\
mystery & \textbf{0.13} & \textbf{0.13} & \textbf{0.17} & \negg{0.10} & \textbf{0.13} & \negg{0.10} & \textbf{0.13} & \negg{0.03} \\
nomystery & 0.20 & \textbf{\possA{0.65}} & \textbf{0.30} & \negg{0.20} & 0.25 & \textbf{\possA{0.50}} & \textbf{0.70} & \negg{0.15} \\
nurikabe & 0.10 & \textbf{\possA{0.45}} & \textbf{0.25} & \negg{0.15} & 0.15 & \textbf{\possA{0.30}} & \textbf{0.40} & \negg{0.15} \\
openstacks & 0.03 & \textbf{\possA{0.57}} & 0.00 & \textbf{\possA{0.63}} & 0.00 & \textbf{\possA{0.63}} & 0.13 & \textbf{\possA{0.48}} \\
optical & 0.00 & \textbf{\possA{0.27}} & 0.00 & \textbf{\possA{0.21}} & 0.00 & \textbf{\possA{0.15}} & \textbf{0.21} & \negg{0.04} \\
parking & \textbf{0.55} & \negg{0.03} & 0.23 & \textbf{\possA{0.65}} & 0.12 & \textbf{\possA{0.65}} & 0.03 & \textbf{\possA{0.72}} \\
pegsol & \textbf{0.90} & \negg{0.10} & \textbf{0.66} & \negg{0.34} & \textbf{0.72} & \negg{0.28} & 0.30 & \textbf{\possA{0.64}} \\
philosophers & 0.38 & \textbf{\possA{0.62}} & \textbf{1.00} & \negg{0.00} & \textbf{1.00} & \negg{0.00} & \textbf{0.98} & \negg{0.02} \\
pipesworld & 0.18 & \textbf{\possA{0.59}} & 0.14 & \textbf{\possA{0.55}} & 0.15 & \textbf{\possA{0.64}} & 0.29 & \textbf{\possA{0.50}} \\
psr & \textbf{0.38} & \negg{0.21} & 0.24 & \textbf{\possA{0.47}} & 0.27 & \textbf{\possA{0.46}} & 0.22 & \textbf{\possA{0.51}} \\
recharging & 0.25 & \textbf{\possA{0.30}} & 0.20 & \textbf{\possA{0.30}} & 0.20 & \textbf{\possA{0.25}} & \textbf{0.35} & \negg{0.25} \\
ricochet & \textbf{0.50} & \negg{0.10} & \textbf{0.55} & \negg{0.10} & \textbf{0.55} & \negg{0.05} & \textbf{0.25} & \negg{0.05} \\
rovers & \textbf{0.40} & \negg{0.35} & \textbf{0.40} & \negg{0.30} & 0.33 & \textbf{\possA{0.38}} & 0.20 & \textbf{\possA{0.57}} \\
rubiks & \textbf{0.70} & \negg{0.10} & \textbf{0.95} & \negg{0.05} & \textbf{0.75} & \negg{0.05} & \textbf{0.45} & \negg{0.00} \\
satellite & \textbf{0.69} & \negg{0.08} & \textbf{0.53} & \negg{0.22} & \textbf{0.61} & \negg{0.14} & 0.03 & \textbf{\possA{0.69}} \\
scanalyzer & \textbf{0.74} & \negg{0.20} & \textbf{0.54} & \negg{0.46} & 0.46 & \textbf{\possA{0.48}} & 0.38 & \textbf{\possA{0.54}} \\
schedule & 0.05 & \textbf{\possA{0.43}} & 0.02 & \textbf{\possA{0.89}} & 0.02 & \textbf{\possA{0.87}} & 0.16 & \textbf{\possA{0.75}} \\
settlers & 0.00 & 0.00 & 0.00 & \textbf{\possA{0.05}} & 0.00 & \textbf{\possA{0.10}} & 0.00 & \textbf{\possA{0.05}} \\
snake & 0.10 & \textbf{\possA{0.25}} & 0.00 & \textbf{\possA{0.45}} & 0.00 & \textbf{\possA{0.50}} & 0.10 & \textbf{\possA{0.35}} \\
sokoban & \textbf{0.82} & \negg{0.14} & \textbf{0.86} & \negg{0.02} & \textbf{0.86} & \negg{0.10} & \textbf{0.74} & \negg{0.10} \\
spider & \textbf{0.45} & \negg{0.35} & 0.35 & \textbf{\possA{0.50}} & 0.30 & \textbf{\possA{0.65}} & 0.25 & \textbf{\possA{0.60}} \\
storage & 0.27 & \textbf{\possA{0.47}} & \textbf{0.40} & \negg{0.10} & 0.23 & \textbf{\possA{0.43}} & \textbf{0.60} & \negg{0.13} \\
termes & \textbf{0.55} & \negg{0.05} & \textbf{0.60} & \negg{0.00} & \textbf{0.55} & \negg{0.05} & \textbf{0.05} & \negg{0.00} \\
tetris & 0.05 & \textbf{\possA{0.65}} & \textbf{0.20} & \negg{0.10} & 0.05 & \textbf{\possA{0.70}} & \textbf{0.65} & \negg{0.00} \\
thoughtful & 0.40 & \textbf{\possA{0.45}} & \textbf{0.40} & \negg{0.30} & 0.40 & \textbf{\possA{0.55}} & \textbf{0.70} & \negg{0.15} \\
tidybot & \textbf{0.60} & \negg{0.25} & \textbf{0.65} & \negg{0.25} & \textbf{0.60} & \negg{0.30} & 0.30 & \textbf{\possA{0.60}} \\
tpp & 0.10 & \textbf{\possA{0.67}} & 0.17 & \textbf{\possA{0.37}} & 0.17 & \textbf{\possA{0.37}} & \textbf{0.70} & \negg{0.10} \\
transport & 0.07 & \textbf{\possA{0.30}} & 0.03 & \textbf{\possA{0.26}} & 0.01 & \textbf{\possA{0.46}} & 0.17 & \textbf{\possA{0.20}} \\
trucks & 0.30 & \textbf{\possA{0.33}} & 0.20 & \textbf{\possA{0.47}} & 0.17 & \textbf{\possA{0.60}} & 0.13 & \textbf{\possA{0.60}} \\
visitall & 0.00 & \textbf{\possA{0.40}} & \textbf{0.07} & \negg{0.00} & 0.00 & \textbf{\possA{0.35}} & \textbf{0.40} & \negg{0.00} \\
woodworking & \textbf{0.70} & \negg{0.24} & \textbf{0.66} & \negg{0.30} & \textbf{0.66} & \negg{0.30} & 0.12 & \textbf{\possA{0.78}} \\
zenotravel & \textbf{0.85} & \negg{0.00} & \textbf{0.70} & \negg{0.25} & \textbf{0.55} & \negg{0.30} & 0.15 & \textbf{\possA{0.80}} \\
\midsplitterB
$\Sigma$ IPC & \textbf{20.88} & \negg{14.73} & \textbf{19.79} & \negg{15.74} & 17.78 & \textbf{\possA{19.55}} & 15.00 & \textbf{\possA{19.65}} \\
\midsplitterB
blocksworld & \textbf{0.05} & \negg{0.00} & \textbf{0.03} & \textbf{0.03} & \textbf{0.03} & \negg{0.00} & 0.00 & \textbf{\possA{0.07}} \\
childsnack & 0.06 & \textbf{\possA{0.12}} & 0.01 & \textbf{\possA{0.42}} & 0.01 & \textbf{\possA{0.44}} & 0.00 & \textbf{\possA{0.43}} \\
genome & \textbf{0.06} & \negg{0.03} & 0.05 & \textbf{\possA{0.11}} & 0.05 & \textbf{\possA{0.10}} & 0.04 & \textbf{\possA{0.11}} \\
labyrinth & 0.05 & \textbf{\possA{0.17}} & \textbf{0.78} & \negg{0.23} & 0.00 & \textbf{\possA{0.23}} & \textbf{0.78} & \negg{0.23} \\
logistics & 0.00 & 0.00 & 0.00 & 0.00 & 0.00 & \textbf{\possA{0.05}} & 0.00 & 0.00 \\
organic & \textbf{0.05} & \negg{0.02} & \textbf{0.07} & \negg{0.02} & \textbf{0.07} & \negg{0.02} & \textbf{0.04} & \negg{0.02} \\
pipesworld & 0.20 & \textbf{\possA{0.38}} & 0.20 & \textbf{\possA{0.24}} & 0.18 & \textbf{\possA{0.40}} & 0.26 & \textbf{\possA{0.34}} \\
rovers & \textbf{0.28} & \negg{0.00} & \textbf{0.23} & \negg{0.00} & \textbf{0.25} & \negg{0.00} & 0.00 & \textbf{\possA{0.07}} \\
visitall & \textbf{0.26} & \negg{0.06} & \textbf{0.17} & \negg{0.16} & 0.13 & \textbf{\possA{0.18}} & 0.06 & \textbf{\possA{0.23}} \\
\midsplitterB
$\Sigma$ HTG & \textbf{1.01} & \negg{0.79} & \textbf{1.53} & \negg{1.19} & 0.72 & \textbf{\possA{1.42}} & 1.18 & \textbf{\possA{1.50}} \\
\midsplitterB
$\Sigma$ & \textbf{21.89} & \negg{15.52} & \textbf{21.32} & \negg{16.93} & 18.50 & \textbf{\possA{20.97}} & 16.18 & \textbf{\possA{21.15}} \\
    \bottomrule
\end{tabularx}

  \caption{Percentage of problems per domain where a column expands fewer nodes than its paired column. The higher value in each pair is indicated in bold. The \mbox{(non-)novelty} extensions are over the \hff{} heuristic.}
  \label{tab:versus-exp-norm}
\end{table}
\fi

\end{document}